\documentclass[letterpaper]{article} % DO NOT CHANGE THIS
\usepackage{aaai2026} 
\usepackage{amsmath,amsfonts,bm}

\def\eqref#1{equation~\ref{#1}}
\def\1{\bm{1}}

\def\rv{{\textnormal{v}}}

\def\rx{{\textnormal{x}}}
\def\ry{{\textnormal{y}}}
\def\rz{{\textnormal{z}}}

\def\rvv{{\mathbf{v}}}

\def\rvx{{\mathbf{x}}}

\def\rvz{{\mathbf{z}}}

\def\vc{{\bm{c}}}

\DeclareMathAlphabet{\mathsfit}{\encodingdefault}{\sfdefault}{m}{sl}
\SetMathAlphabet{\mathsfit}{bold}{\encodingdefault}{\sfdefault}{bx}{n}

\newcommand{\E}{\mathbb{E}}

\DeclareMathOperator*{\argmax}{arg\,max}

\usepackage{times}  % DO NOT CHANGE THIS
\usepackage{helvet}  % DO NOT CHANGE THIS
\usepackage{courier}  % DO NOT CHANGE THIS
\usepackage[hyphens]{url}  % DO NOT CHANGE THIS
\usepackage{graphicx} % DO NOT CHANGE THIS
\usepackage{natbib}  % DO NOT CHANGE THIS AND DO NOT ADD ANY OPTIONS TO IT
\usepackage{caption} % DO NOT CHANGE THIS AND DO NOT ADD ANY OPTIONS TO IT
\usepackage{algorithm}
\usepackage{algorithmic}

\usepackage{newfloat}
\usepackage{listings}
\DeclareCaptionStyle{ruled}{labelfont=normalfont,labelsep=colon,strut=off} % DO NOT CHANGE THIS
\floatstyle{ruled}
\newfloat{listing}{tb}{lst}{}
\floatname{listing}{Listing}
\usepackage{url}
\usepackage{amssymb}
\usepackage{amsthm}
\usepackage{apxproof}
\usepackage{thmtools, thm-restate}
\usepackage{comment}
\usepackage{tikz}
\usepackage{array}
\usepackage{multirow}
\usepackage{listings}
\usepackage{float}
\usepackage{booktabs}

\definecolor{bgcolor}{rgb}{0.95,0.95,0.92}
\definecolor{codegreen}{rgb}{0,0.6,0}
\definecolor{codegray}{rgb}{0.5,0.5,0.5}
\definecolor{codepurple}{rgb}{0.58,0,0.82}
\definecolor{codeblue}{rgb}{0.25,0.5,0.75}
\definecolor{codered}{rgb}{0.8,0.2,0.2}
\definecolor{darkviolet}{rgb}{0.58,0,0.82}

\usepackage{newfloat}
\usepackage{listings}
\DeclareCaptionStyle{ruled}{labelfont=normalfont,labelsep=colon,strut=off} % DO NOT CHANGE THIS

\floatstyle{ruled}
\newfloat{listing}{tb}{lst}{}
\floatname{listing}{Listing}

\title{Cliqueformer: Model-Based Optimization with Structured Transformers}

\author{
    Jakub Grudzien Kuba\textsuperscript{\rm 1},
    Pieter Abbeel\textsuperscript{\rm 1},
    Sergey Levine\textsuperscript{\rm 1}\\
}
\affiliations{
    \textsuperscript{\rm 1}BAIR, UC Berkeley\\
    \{kuba, pabbeel, sergey.levine\}@berkeley.edu
}

\begin{document}

\maketitle

\begin{abstract}
Large neural networks excel at prediction tasks, but their application to design problems, such as protein engineering or materials discovery, requires solving offline model-based optimization (MBO) problems.
%%SL.1.30: can we at least briefly explain what MBO problems are? many readers won't know
While predictive models may not directly translate to effective design, recent MBO algorithms incorporate reinforcement learning and generative modeling approaches.
%%SL.1.30: it's not really clear what this sentence means or what is its significance. Why do predictive models not translate to effective design? what is the relevance of recent MBO algorithms using RL and generative models? what are recent MBO algorithms?
Meanwhile, theoretical work suggests that exploiting the target function's structure can enhance MBO performance.
%%SL.1.30: seems vague, I know you are referring to FMB, but at this stage in the paper no one reading this will understand that
We present Cliqueformer, a transformer-based architecture that learns the black-box function's structure through functional graphical models (FGM), addressing distribution shift without relying on explicit conservative approaches. Across various domains, including chemical and genetic design tasks, Cliqueformer demonstrates superior performance compared to existing methods. 
We open-source our code at \url{https://github.com/znowu/cliqueformer-code}.
\end{abstract}

\section{Introduction}
Most of the common use cases of deep learning (DL) so far have taken the form of prediction tasks \citep{hochreiter1997long, he2016deep, krizhevsky2017imagenet}.
However, in many applications, such as protein synthesis or chip design, we might want to use powerful
models to instead solve \emph{optimization} problems. Clearly, accurate predictions of a target score of an object could be used to find a design of the object that maximizes that score.
Such a methodology is particularly useful in engineering problems in which evaluating solution candidates comes with big risk. For example, synthesizing a proposed protein requires a series of wet lab experiments and induces extra cost and human effort \citep{gomez2018automatic, brookes2019conditioning}. Thus, to enable proposing \emph{de-novo} generation of strong solution candidates, researchers have drawn their attention to offline \emph{model-based optimization} (MBO). In this paradigm, first, a surrogate model of the score is learned from offline data. Next, a collection of designs is trained to maximize the surrogate, and then proposed as candidates for maximizers of the target score \citep{gomez2018automatic, kumar2021data}.

Unfortunately, MBO introduces unique challenges not encountered in classical prediction tasks.
%%SL.1.30: I would try to cover the challenge in paragraph 1, since stating what is the problem/challenge is really important to motivate the method
The most significant issue arises from the incomplete coverage of the design space by the data distribution.
%%SL.1.30: this is a slightly weird way to phrase it, implies that if we had more coverage the problem would go away (kind of true, but not in a useful way...)
This limitation leads to a phenomenon known as \emph{distribution shift}, where optimized designs drift away from the original data distribution. Consequently, this results in poor proposals with significantly overestimated scores \citep{trabucco2022design, geng2023offline}. To address it, popular MBO algorithms have been employing techniques from offline reinforcement learning \citep{kumar2020conservative, trabucco2021conservative} and generative modeling \citep{kumar2020model, mashkaria2023generative} to enforce the in-distribution constraint. Meanwhile, much of the recent success of DL has been driven by domain-specific neural networks that, when scaled together with the amount of data, lead to increasingly better performance. While researchers have managed to establish such models in several fields, it is not immediately clear how to do it in MBO. Recent theoretical work, however, has shown that MBO methods can benefit from information about the target function's \emph{structure} expressed by its \textit{functional graphical model} \citep[FGM]{grudzien2024functional}. This insight opens up new possibilities for developing more effective MBO models by injecting such structure into their architecture. However, how to integrate such decompositions into scalable neural networks remains an open question, and addressing this challenge is the focus of this work.

In contrast to previous works, in our paper, we develop a scalable model that tackles MBO by learning the structure of the black-box function through the formalism of functional graphical models. Our architecture aims to solve MBO by \emph{1)} decomposing the predictions over the \emph{cliques} of the function's FGM, and \emph{2)} enforcing the cliques' marginals to have wide coverage with our novel form of the variational bottleneck \citep{kingma2013auto, alemi2016deep}. However, building upon our Theorem \ref{lemma:rotation}, we do not follow \cite{grudzien2024functional} during the FGM discovery step, and instead subsume it in the learning algorithm. To enable scaling to high-dimensional problems and large datasets, we employ the transformer backbone \citep{vaswani2017attention}. Empirically, we show that our model, \emph{Cliqueformer}, inherits the scalability guarantees of MBO with FGM. We further demonstrate its superiority to baselines in a suite of tasks with latent radial-basis functions \citep{grudzien2024functional} and real-world chemical \citep{hamidieh2018data} and DNA design tasks \citep{trabucco2022design, uehara2024understanding}.

\section{Preliminaries}
In this section, we provide the necessary background on offline model-based optimization. Additionally, we cover the basics of functional graphical models on top of which we build Cliqueformer.

\subsection{Offline Model-based Optimization}
We consider an offline model-based optimization (MBO for short) problem, where we are given a dataset $\mathcal{D}=\{\rvx^{i}, \ry^{i}\}_{i=1}^{N}$ of examples $\rvx$ from a \emph{design space} $\mathcal{X}$, following distribution $p(\rvx)$, and their values $\ry=f(\rvx) \in \mathbb{R}$ under an unkown (black-box) function $f(\rvx)$. 
Our goal is to optimize this function \textbf{offline}. That is, to find its maximizer
\begin{align}
    \label{eq:bbo}
    \rvx^{\star} = \argmax_{\rvx\in\mathcal{X}} f(\rvx)
\end{align}
by only using information provided in $\mathcal{D}$ \citep{kumar2020model}. 
Sometimes, a more general objective in terms of a \emph{policy} over $\pi(\rvx)$ is also used, $\eta(\rvx) = \E_{\rvx\sim\pi}[f(\rvx)]$.
In either form, unlike in Bayesian optimization, we cannot make additional queries to the black-box function \citep{brochu2010tutorial, kumar2020model}. 
This formulation represents settings in which obtaining such queries is prohibitively costly, such as tests of new chemical molecules or of new hardware architectures \citep{kim2016pubchem, kumar2021data, yang2024generative}.

To solve MBO, it is typical to learn a model $f_{\theta}(\rvx)$ of $f(\rvx)$ parameterized by a vector $\theta$ with a regression method and data from $\mathcal{D}$,
\begin{align}
    \label{eq:model}
    L(\theta) = \E_{(\rvx, \ry)\sim\mathcal{D}}\big[ \big(f_{\theta}(\rvx) - \ry \big)^2\big] + \text{\normalfont Reg}(\theta) 
\end{align}
where $\text{\normalfont Reg}(\theta)$ is an optional regularizer. 
Classical methods choose $\text{\normalfont Reg}(\theta)$ to be identically zero, while conservative methos use the regularizer to bring the values of examples out of $\mathcal{D}$ down. 
For example, the regularizer of Conservative Objective Model's \citep[COMs]{trabucco2021conservative} is 
\begin{align}
    \text{Reg}_{\text{com}}(\theta) = \alpha \big(\E_{\rvx \sim \mu_{\theta_{\perp}}}[f_{\theta}(\rvx)] - \E_{\rvx \sim \mathcal{D}}[f_{\theta}(\rvx)]\big), \quad \alpha > 0,\nonumber
\end{align}
where $(\cdot)_{\perp}$ is the stop-gradient operator
and $\mu_{\theta_{\perp}}(\rvx)$ is the distribution obtained with a few gradient ascent steps on $\rvx$ initialized from $\mathcal{D}$. This distribution depends on the value of $\theta$ but is not differentiated through while computing the loss's gradient, and thus we denote it by $\theta_{\perp}$.
\begin{figure}[t]
    \centering
    % Subfigure (a)
    \hspace{-2em}
    \begin{minipage}[b]{0.25\textwidth}
        \centering
        \includegraphics[width=\linewidth]{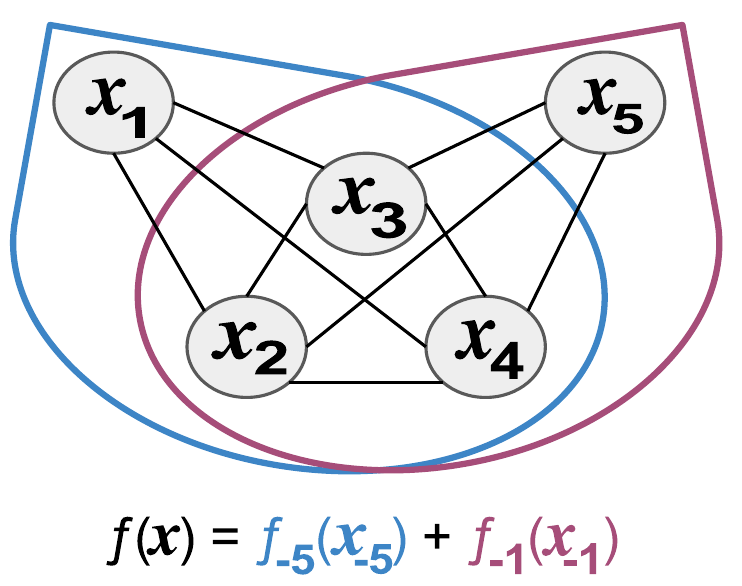}
    \end{minipage}
    % Subfigure (b)
    \begin{minipage}[b]{0.25\textwidth}
        \centering
        \includegraphics[width=\linewidth]{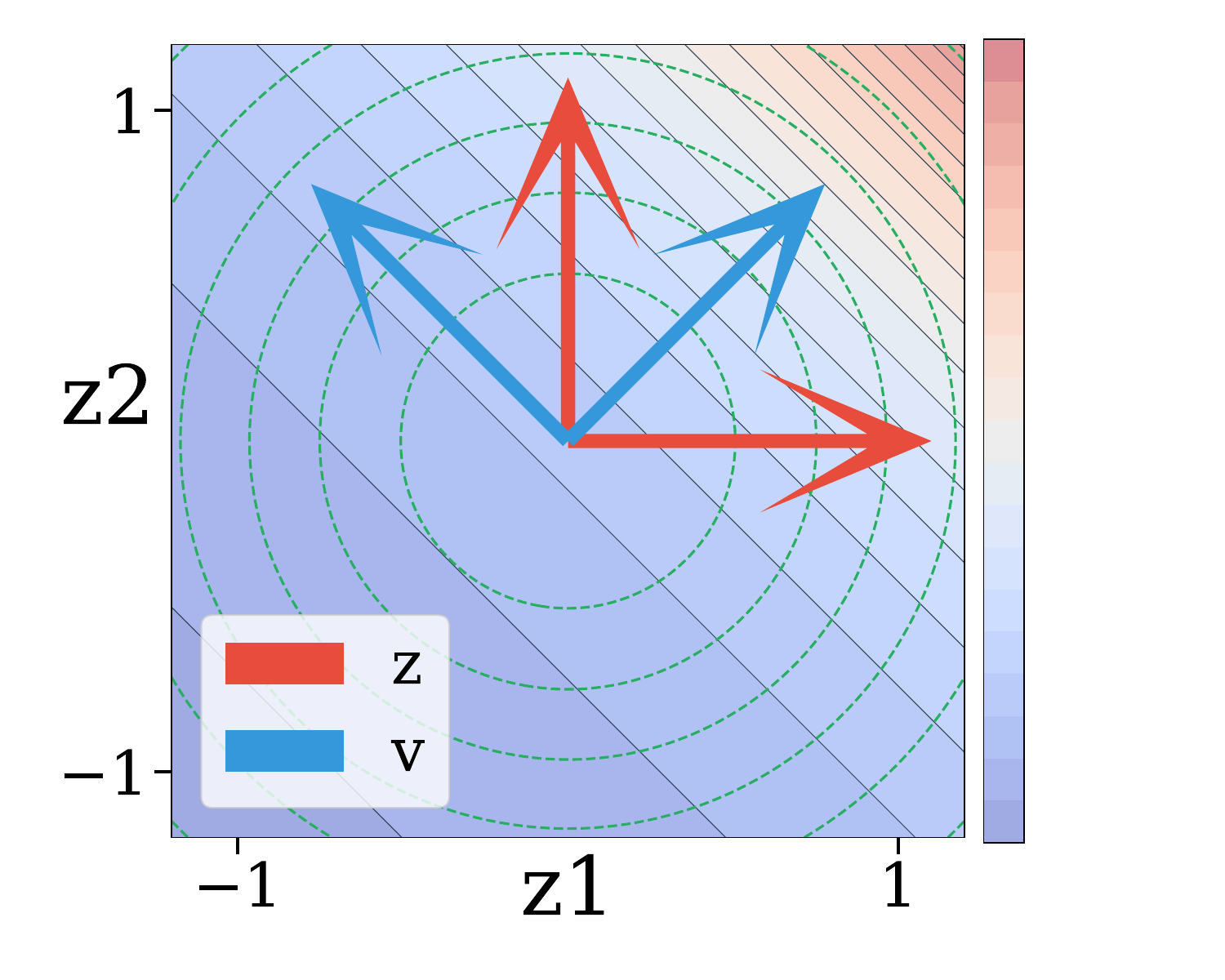}
    \end{minipage}
    \hspace{-4em}
    \caption{\emph{Left.} An FGM of a 5D function which decomposes as $f(\rvx) = f_{-5}(\rvx_{-5}) + f_{-1}(\rvx_{-1})$. By Definition~\ref{def:fgm}, nodes $\rx_1$ and $\rx_5$ are not linked. \emph{Right.} Illustration of construction in the proof of Theorem~\ref{lemma:rotation} in 2D. Red axes ($\rvz$) show contour lines of $f$ depending on both coordinates $\rz_1$ and $\rz_2$. Blue axes ($\rvv$) show the same contours depending only on $\rv_1$ after rotation. The Gaussian distribution (green circles) maintains circular shape in both coordinate systems, demonstrating invariance under rotation.}
    \label{fig:fgm-rotation}
\end{figure}

Unfortunately, in addition to the extra computational cost that the inner-loop gradient ascent induces, COMs's regularizer limits the amount of improvement that it allows its designs to make.
This is particularly frustrating since recent work on functional graphical models \citep[FGM]{grudzien2024functional} delivered a premise of large improvements in the case when the black-box function's graph can be discovered, as we explain in the next subsection.

\subsection{Functional Graphical Models}

An FGM of a high-dimensional function $f(\rvx)$ is a graph over components of $\rvx$ that separates $\rx_i, \rx_j \in \rvx$ if their contributions to $f(\rvx)$ are independent of each other. 
Knowing such a structure of $f$ one can eliminate interactions between independent variables from a model that approximates it, and thus prevent an MBO algorithm from adversarially exploiting them.
We summarize basic properties of FGMs below. 
In what follows, we denote $\mathcal{X}_{-i}$ as the design (input) space without the $i^{th}$ subspace, and $\rvx_{-i}$ as its element \footnote{For example, if $\mathcal{X}=\mathcal{X}_{1} \times \mathcal{X}_{2} \times \mathcal{X}_{3}$ and $\rvx=(\rx_1, \rx_2, \rx_3)$, then $\mathcal{X}_{-2} = \mathcal{X}_{1} \times \mathcal{X}_{3}$, and $\rvx_{-2}=(\rx_1, \rx_3)$.}.
%\begin{tcolorbox}[theorembox]
\begin{restatable}{definition}{fgm}
    \label{def:fgm}
    Let $\rvx = (\rx_v \ | \ v \in \mathcal{V})$ be a joint variable with index set $\mathcal{V}$, and $f(\rvx)$ be a real-valued function.
    An FGM $\mathcal{G}=(\mathcal{V}, \mathcal{E})$ of $f(\rvx)$ is a graph where the edge set $\mathcal{E} \subset \mathcal{V}^2$ is such that,
\begin{center}
    $\exists \   f_{-i}:\mathcal{X}_{-i} \rightarrow  \mathbb{R} \  \text{and} \  \
    f_{-j}:\mathcal{X}_{-j} \rightarrow \mathbb{R}, 
    \  \text{with}  \ 
    f(\rvx) = f_{-i}(\rvx_{-i}) + f_{-j}(\rvx_{-j}),
    \text{implies} \ (i, j) \notin \mathcal{E}.$
\end{center}
See Figure \ref{fig:fgm-rotation} for illustration.
\end{restatable}    
%\end{tcolorbox}

%The basic result about FGMs is that they allow for decomposition of the target function into sub-functions with smaller, partially-overlapping inputs, from the FGM's set of maximal cliques $\mathcal{C}$,

Crucially, FGMs enable decomposing the target function into sub-functions with smaller inputs defined by the set of maximal cliques $\mathcal{C}$ of the FGM's graph,
\begin{align}
\label{eq:decomp}
f(\rvx) = \sum_{C\in\mathcal{C}}f_{C}(\rvx_C).
\end{align}

Intuitively, the decomposition enables more efficient learning of the target function since it can be constructed by adding together functions defined on smaller inputs, which are easier to learn.
This, in turn, allows for more efficient MBO since the joint solution $\rvx^{\star}$ can be recovered by \emph{stitching} individual solutions $\rvx^{\star}_{C}$ to smaller problems.
This intuition is formalized by the following theorem.

%\begin{tcolorbox}[theorembox]
\begin{restatable}[\cite{grudzien2024functional}]{theorem}{regret}
    \label{th:regret}
    For a real-valued function $f(\rvx)$ with FGM with maximal cliques $\mathcal{C}$, and policy class $\Pi$, the regret of MBO with FGM information satisfies
    \begin{center}
        $\eta(\pi^{\star}) - \eta(\hat{\pi}_{\text{FGM}}) \leq \text{C}_{\text{stat}} \text{C}_{\text{cpx}} \max\limits_{\pi\in\Pi, \rvx \in \mathcal{X}, C\in\mathcal{C}} \frac{\pi(\rvx_{C})}{p_{C}(\rvx_C)}$
    \end{center}
    where $\text{C}_{\text{stat}}$, $\text{C}_{\text{cpx}}$ are distribution and complexity constants defined in Appendix B.
\end{restatable}   
%\end{tcolorbox}

This theorem  implies that a function approximator equipped with FGM does not require the dataset to cover the entire design space. Instead, it suffices for the dataset to span individual cliques of the space, which is a substantially less stringent condition—particularly when the cliques are small. To illustrate this result numerically, we conduct an experiment, shown in the left column of Figure \ref{fig:lil-exp}. Specifically, we generate high-dimensional Gaussian data $\rvx$ and evaluate $f(\rvx)$, which is modeled as a mixture of random radial basis functions (RBFs) constrained by an FGM of triangles. We then train two models, $f_{\theta}(\rvx)$, each with the same parameter count: one adhering to the FGM decomposition and the other not. These models are used both to approximate the target function and to optimize $\rvx$. The results demonstrate that the FGM-equipped model significantly outperforms its counterpart, both in terms of the quality of fit and, more critically, in the quality of the optimized designs.

\begin{figure}[t]
  \centering

  \begin{minipage}[t]{0.235\textwidth}
    \centering
    \includegraphics[width=\linewidth]{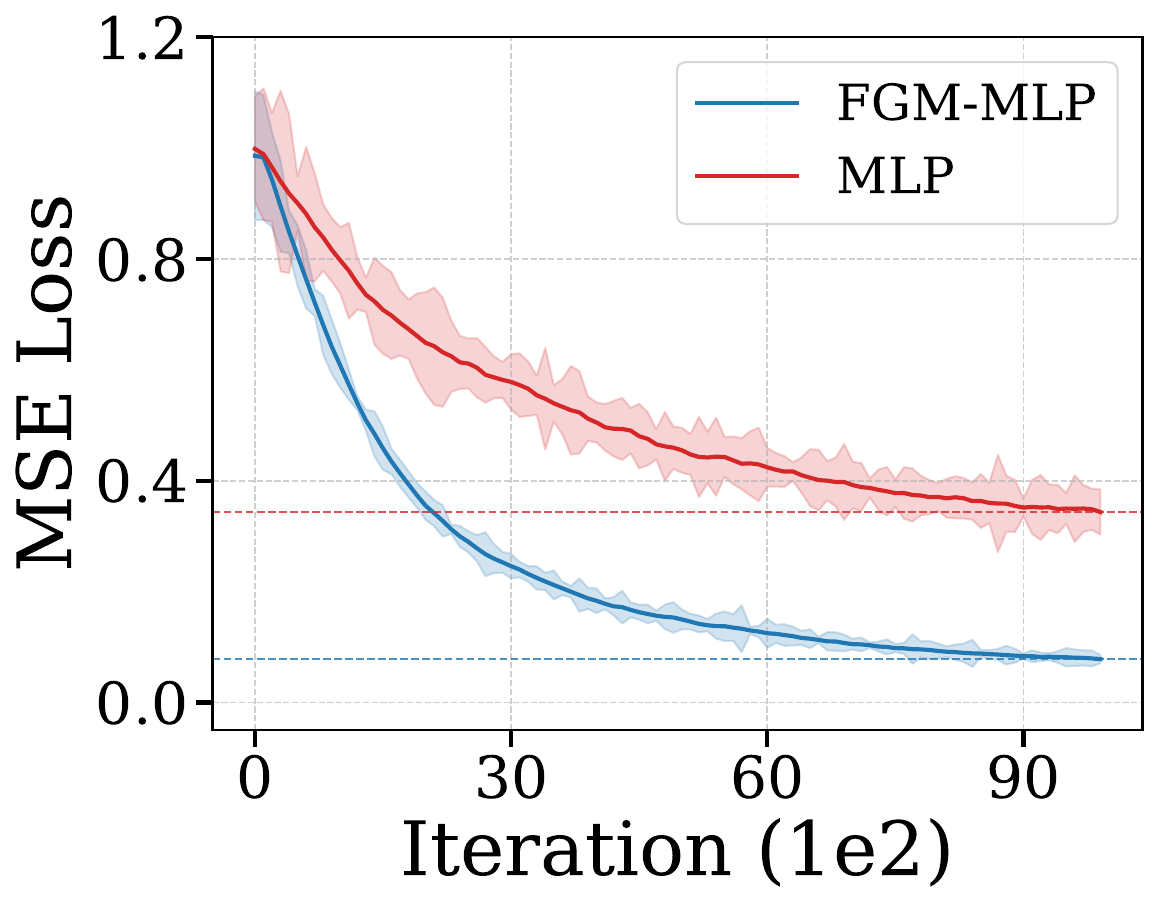}\\[-2pt]
    (a) MSE loss (FGM)
  \end{minipage}\hfill
  \begin{minipage}[t]{0.235\textwidth}
    \centering
    \includegraphics[width=\linewidth]{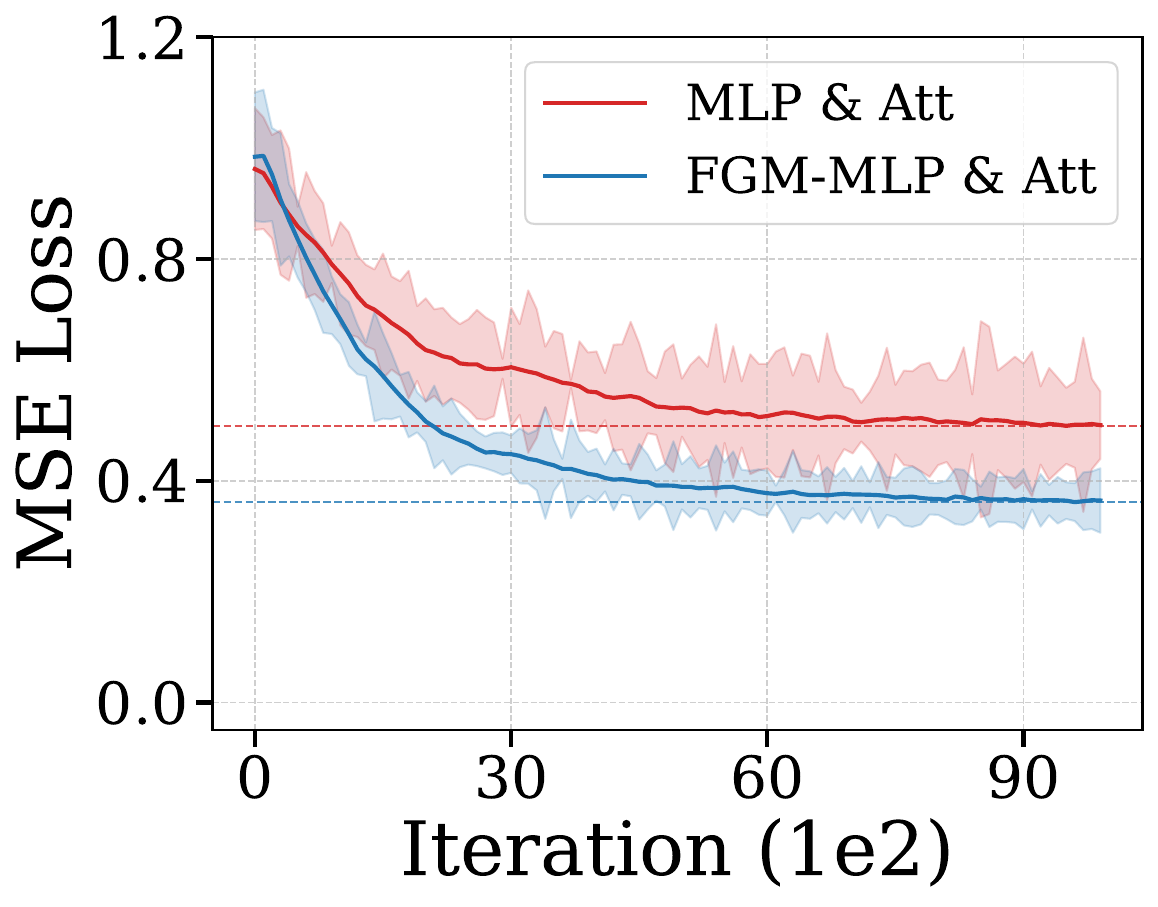}\\[-2pt]
    (b) MSE loss (no FGM)
  \end{minipage}

  \begin{minipage}[t]{0.235\textwidth}
    \centering
    \includegraphics[width=\linewidth]{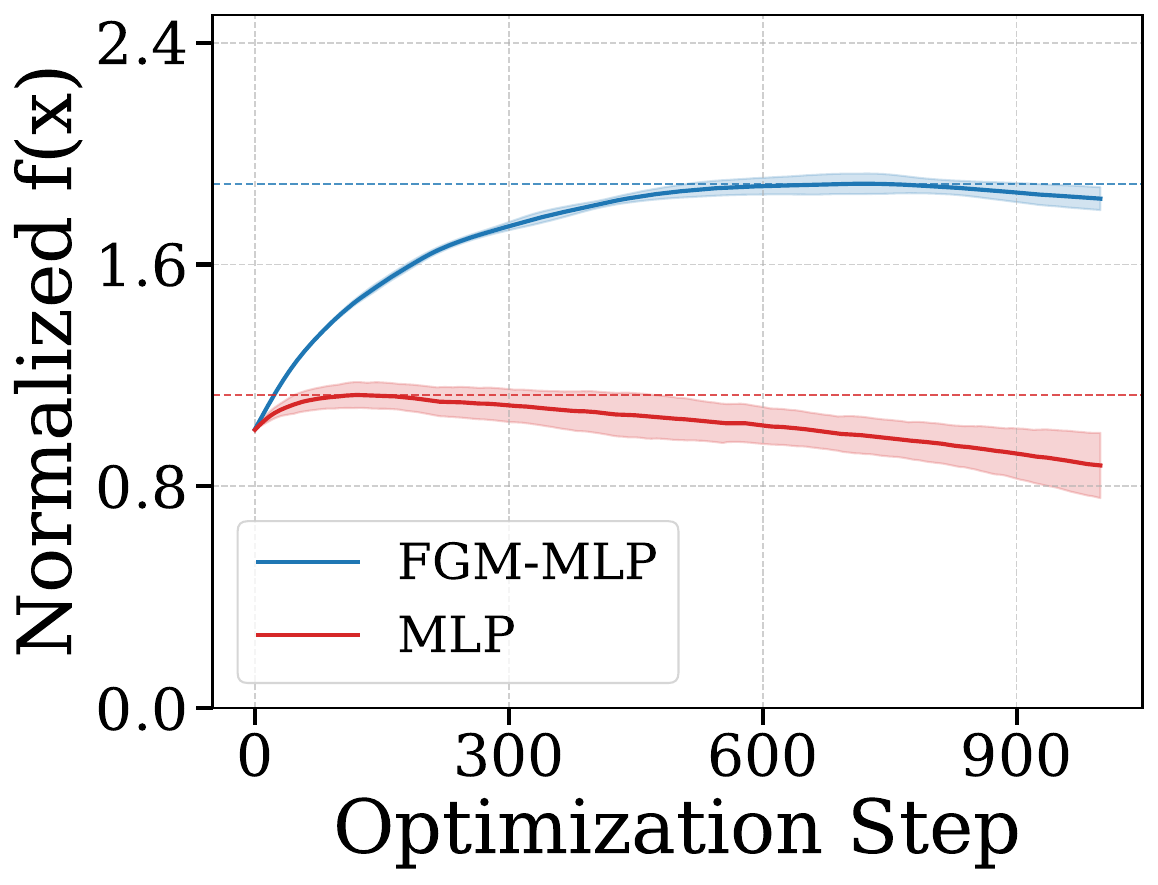}\\[-2pt]
    (c) f(x) value (FGM)
  \end{minipage}\hfill
  \begin{minipage}[t]{0.235\textwidth}
    \centering
    \includegraphics[width=\linewidth]{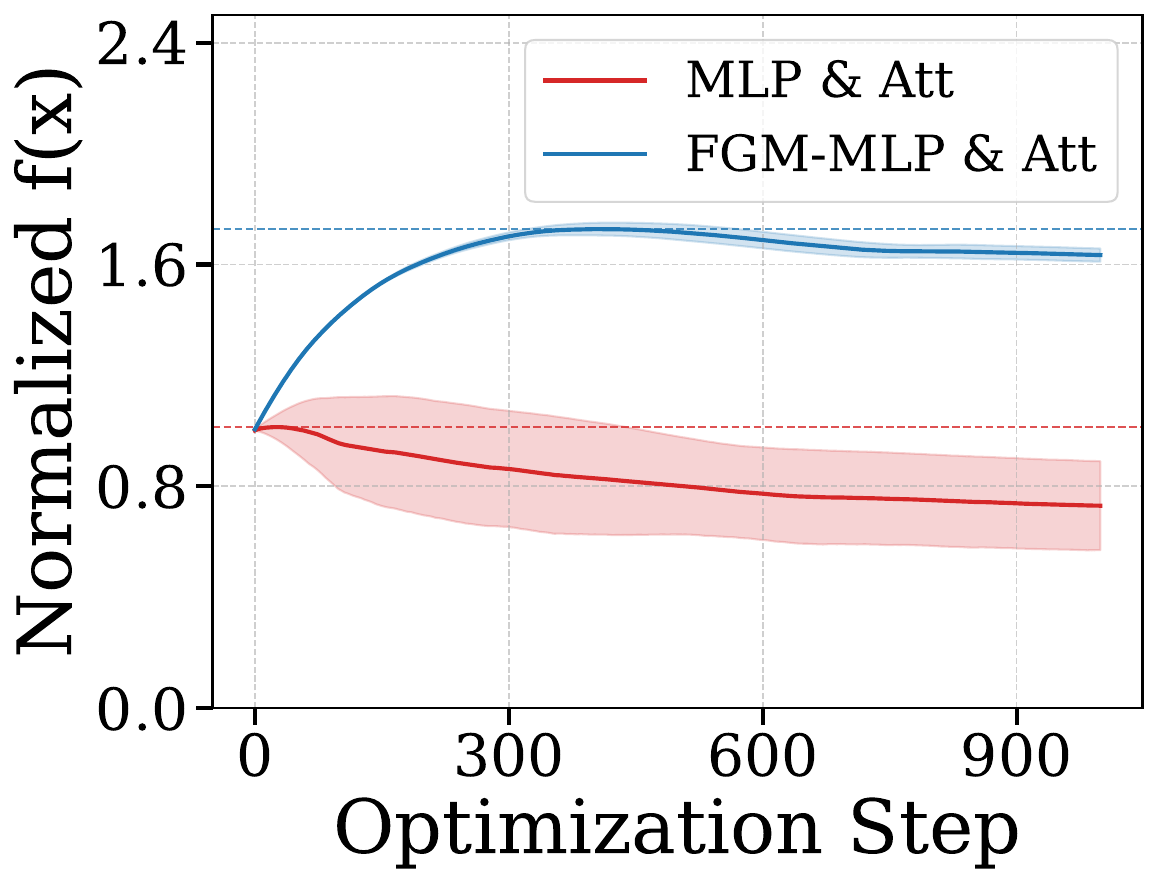}\\[-2pt]
    (d) f(x) value (no FGM)
  \end{minipage}

  \caption{{\emph{Left column. FGM known. } The model without the FGM information (red) fits the data poorly (Fig. (a)) and leads to poor designs (Fig. (b)). The model with FGM information (blue) achieves a good fit and leads to designs largely superior to the data. \emph{Right column. FGM unknown.} The model with \emph{an} FGM decomposition (blue) achieves a slightly better fit than an oblivious model (red) but leads to significantly better designs.
    }}
  \label{fig:lil-exp}
\end{figure}

In the next section, we show how these results can be combined with a transformer, mitigate the distribution shift problem, and enable efficient MBO.

\section{Cliqueformer}
\label{sec:cliqueformer}

This section introduces a neural network model to solve MBO problems through standard end-to-end training on offline datasets. We present a new theoretical result, outline the key desiderata for such a model, and propose an architecture—\emph{Cliqueformer}—that addresses these requirements.

\subsection{Structure Discovery}
\label{subsec:non}

The regret bound from Theorem \ref{th:regret} applies to methods that use the target function's FGM in their function approximation.
It implies that such methods can solve even very high-dimensional problems if their underlying FGMs have low-dimensional cliques or, simply speaking, are \emph{sparse}.
Since, in general, no assumptions about the input can be made, this motivates learning a representation of the input for which one can make distributional assumptions and infer the FGM with statistical tests.
Following this reasoning, \citet{grudzien2024functional} offer a heuristic technique for discovering an FGM over learned, latent, normally-distributed variables.
However, as we formalize with the following theorem, which we prove in Appendix B, even such attempts are futile in dealing with black-box functions.

%\begin{tcolorbox}[theorembox]
\begin{restatable}{theorem}{rotation}
    \label{lemma:rotation}
    For any $d \geq 2$ and random variable $\rvx\in\mathbb{R}^{d}$ with positive probability density, there exists a function $f(\rvx)$ and two standard-normal reparameterizations $\rvz = z(\rvx)$, $\rvv = v(\rvx)$ such that the FGM of $f$ is complete with respect to $\rvz$ but empty with respect to $\rvv$.
\end{restatable}    
%\end{tcolorbox}

The theorem (see Figure \ref{fig:fgm-rotation} for intuition behind the proof) implies that FGM is not a fixed attribute of a function that can be estimated from the data, but instead should be viewed as a property of the input's reparameterization.
Furthermore, different reparameterizations feature different FGMs with varied levels of decomposability, some of which may not significantly simplify the target function.
This motivates a reverse approach that begins from defining a desired FGM and learning representations of the input that align with the graph.
To examine if FGM decomposition implemented this way can still bring benefits in MBO, we conduct another experiment in the right column of Figure \ref{fig:lil-exp}. 
This time, we conceal the FGM information from both models, while implementing one of them with another decomposition, in a latent space that is computed with an attention layer (for fairness, we gave the FGM-oblivious model an attention layer too). 
Despite not bringing major improvements to the quality of fit, the decomposable model does lead to better designs, providing empirical support for the considered approach. 
Consequently, in the next subsection, we introduce Cliqueformer, where the FGM is specified as a hyperparameter of the model and a representation of the data that follows its structure is learned.

\subsection{Architecture}
The goal of this subsection is to derive an MBO model
that can simultaneously learn the target function as well as its structure, and thus be readily applied to MBO.
First, we would like the model to decompose its prediction into a sum of models defined over small subsets of the input variables in the manner of Equation (\ref{eq:decomp}).
As discussed in the previous subsection, in general settings, while such a structure is not known a priori, a sufficiently expressive model can learn to follow a pre-defined structure.
Thus, instead, we propose that the FGM be defined first, and a representation of the data be learned to align with the chosen structure.

% Use tcolorbox to wrap the restatable
%\begin{tcolorbox}[theorembox]
    \begin{restatable}{desideratum}{desOne}\label{des1}
The model should follow a pre-defined FGM decomposition in a learned space.
\end{restatable}
%\end{tcolorbox}

Accomplishing this desideratum is simple.
A model that we train should, after some transformations, split the input's representations, denoted as $\rvz$, into partially-overlapping cliques, $(\rvz_{C_1}, \dots , \rvz_{C_{N_{\text{clique}}}})$,
and process them independently, followed by a summation.
To implement this, we specify how many cliques we want to decompose the function into, $N_{\text{clique}}$, as well as their dimensionality, $d_{\text{clique}}$, and the size of their \emph{knots}---dimensions at which two consecutive cliques overlap, $d_{\text{knot}}=|\rvz_{C_i} \cap \rvz_{C_{i+1}}|$.
Then, we pass each clique through an MLP network that is also equipped with a trigonometric clique embedding, similar to \citet{vaswani2017attention} to express a different function for each clique,
\begin{center}
    $f_{\theta}(\rvz) = \frac{1}{N_{\text{clique}}} \sum_{i=1}^{N_{\text{clique}}} f_{\theta}(\rvz_{C_i}, \vc_i)$, \\ 
    \text{where} \ 
    $\vc_{i, 2j}, \vc_{i, 2j+1} = \sin(i \cdot \omega_j), \cos(i \cdot \omega_j)$
\end{center}
and $\omega_j = 10^{-8j/d_{\text{model}}}$.
Here, we use the arithmetic mean over cliques rather than summation because, while being functionally equivalent, it provides more stability when $N_{\text{clique}} \rightarrow \infty$.
Pre-defining the cliques over the representations allows us to avoid the problem of discovering arbitrarily dense graphs, as explained in Subsection \ref{subsec:non}.
This architectural choice implies that the regret from Theorem \ref{th:regret} will depend on the coverage of such representations' cliques, $\max_{i\in [N_{\text{clique}}]}1/e_{\theta}(\rvz_{C_i})$, where $e_{\theta}(\rvz) = \mathbb{E}_{\rvx\sim \mathcal{D}}[e_{\theta}(\rvz|\rvx)]$ is the marginal distribution of the representations learned by an encoder $e_{\theta}$.
This term can become dangerously large if individual distributions $e_{\theta}(\rvz_{C})$ 
put disproportionally more density to some regions of the latent space than to others.
Thus, to prevent that, we propose to train the latent space so that the distribution of cliques attain broad coverage\footnote{For a probability distribution of z to attain a \emph{broad coverage} means to distribute its
probability mass on a large set of values of z, as opposed to be tightly concentrated around its modes.}.
%\begin{tcolorbox}[theorembox]
    \begin{restatable}{desideratum}{des2}
    Individual cliques of the learned representations should attain broad coverage.
\end{restatable}
%\end{tcolorbox}

To meet this requirement, we leverage tools from representation learning, but in a novel way.
Namely, we put a variational bottleneck \citep[VIB]{kingma2013auto, higgins2016beta, alemi2016deep} on \textbf{individual} cliques of our representations that brings their distribution closer to a prior with wide coverage, which we choose to be the standard-normal prior.
To implement it, when computing the loss for a single example, we sample a single clique to compute the VIB for at random, as opposed to computing it for the joint latent variable like in the classical VIB,
\begin{align}
    %\label{eq:vib}
    \text{VIB}(\rvx, \theta) = \E_{i \sim U[N_{\text{clique}}]}\big[ \text{KL}\big(e_{\theta}(\rvz_{C_i}|\rvx), p_{C_i}(\rvz_{C_i}) \big) \big], \nonumber
\end{align}
where $e_{\theta}(\rvz_{C}|\rvx)$ is the density of clique $C$ produced by our learnable encoder $e_{\theta}(\rvz|\rvx)$, and $p_{C}(\rvz_C)$ is the density of $\rvz_C$ under the standard-normal distrituion.
Note that it is not equivalent to the classical, down-weighted VIB in expectation either since our cliques overlap, meaning that knots contribute to the VIB more often than regular dimensions.
We ablate for the impact of this term in Appendix C.

Lastly, since in practical problems, such as biological and chemical design, the data is often discrete \citep{uehara2024understanding, yang2024generative}, the designs cannot always be readily optimized with gradient-based methods.
Thus, to be generally applicable, our model should be equipped with a data type-agnostic, approximately invertible map from the data space to a continuous space, where gradient-based optimization is feasible.
This brings us to our last desideratum.
%\begin{tcolorbox}[theorembox]
    \begin{restatable}{desideratum}{des3}
    The model should have a decoder that maps representations back to the original data space.
\end{restatable}
%\end{tcolorbox}

Thus, together with the model $f_{\theta}(\rvz)$ and the encoder $e_{\theta}(\rvz|\rvx)$, we train a decoder $d_{\theta}(\rvx|\rvz)$ that reconstructs the designs from the latent variables.
Putting it all together, the training objective of our model is a VIB-style likelihood objective with a regression term,
\begin{align}
    \label{eq:loss}
    L_{\text{clique}}(\theta) &= \E_{(\rvx, \ry) \sim \mathcal{D}, \rvz \sim e_{\theta}(\cdot|\rvx)}\Big[ 
    \text{VIB}(\rvx, \theta) \nonumber\\
    &- \log d_{\theta}(\rvx|\rvz) 
    + \tau \cdot \big(\ry - f_{\theta}(\rvz)\big)^2 \Big], 
\end{align}
where $\tau$ is a positive coefficient that we set to $10$.

\begin{figure*}[t]
\begin{center}
    \includegraphics[width=6in]{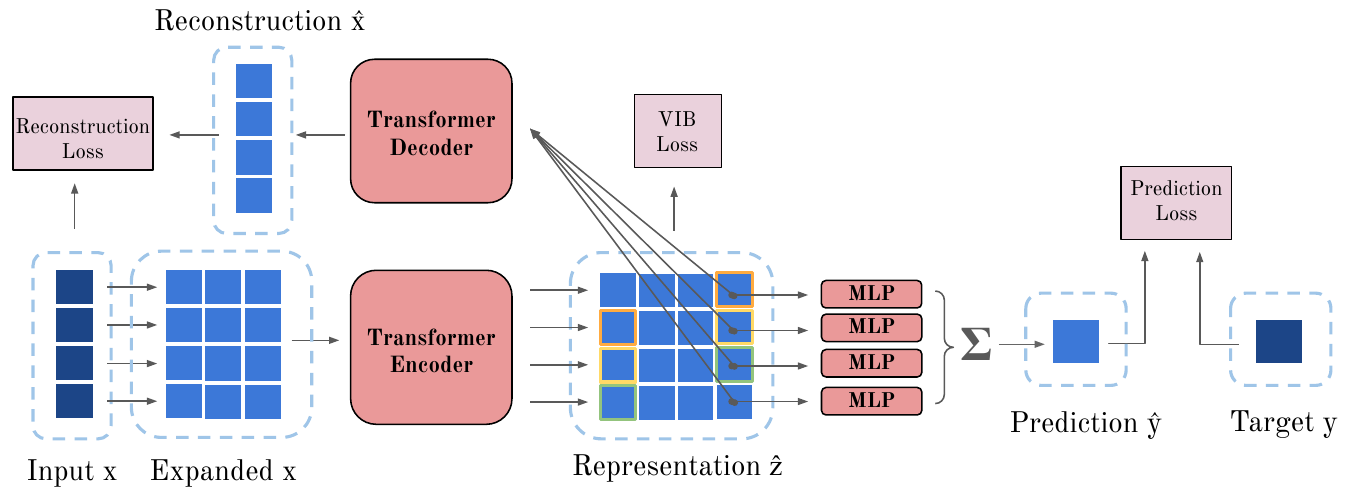}
\caption{
%%SL.9.27: see my comment in the main text about improving this figure, it's a good start but needs a lot of work
{Illustration of information flow in Cliqueformer's training.
Data are shown in navy, learnable variables in blue, neural modules in red, and loss functions in pink.
The variables in input $\rvx$ are expanded into high-dimensional vectors (Expanded $\rvx$) and, treated as a sequence of embeddings, passed to a transformer encoder to compute representation $\rvz$ which is decomposed into cliques.
The representation goes to the parallel MLPs whose outputs, added together, predict target $\ry$.
The representation $\rvz$ is also fed to a transformer decoder that tries to recover the original input $\rvx$.
Additionally, the representation goes through an information bottleneck during training.}  
}
 \label{fig:cliqueformer}
\end{center}
 \end{figure*} 

We use transformer networks \citep{vaswani2017attention} for both encoder $e_{\theta}$ and decoder $d_{\theta}$ to achieve the expressivity needed for FGM decomposition in latent space. The encoder transforms input $\rvx \in \mathbb{R}^{d}$ into $d$ vectors of size $d_{\text{model}}$, processes them as token embeddings, and outputs a normal distribution $e_{\theta}(\rvz|\rvx)$. Sampled representations $\rvz$ are arranged into $N_{\text{clique}}$ cliques (dimension $d_{\text{clique}}$, knot size $d_{\text{knot}}$) before being passed to both the predictive model $f_{\theta}$ and transformer decoder $d_{\theta}$.
For illustration of the information flow in Cliqueformer's training consult Figure \ref{fig:cliqueformer}.

\begin{algorithm}[t!]
\caption{MBO with Cliquefrormer}
\label{algorithm:bbo-cliqueformer}
\begin{algorithmic}[1]
\STATE Initialize the encoder, decoder, and predictive model $(e_{\theta}, d_{\theta}, f_{\theta})$.
\FOR{$t=1, \dots, T_{\text{model}}$}
    \STATE Take a gradient step on the parameter $\theta$ with respect to $L_{\text{clique}}(\theta)$ from Equation (\ref{eq:loss}).
\ENDFOR
\STATE Sample $B$ examples $\rvx^{(i_b)} \sim \mathcal{D}$ from the dataset.
\STATE Encode the examples  $\rvz^{(i_b)} \sim e_{\theta}(\rvz|\rvx^{(i_b)})$,. 
\FOR{$t=1, \dots, T_{\text{design}}$}
    \STATE Decay the representations, $\rvz^{(i_b)} \gets (1-\lambda) \rvz^{(i_b)}$.
    \STATE Take a gradient ascent step on the parameter $\rvz$ with respect to $L_{\text{mbo}}$ from Equation (\ref{eq:mbo-objective}).
\ENDFOR
\STATE Propose solution candidates by decoding the representations, $\rvx^{\star} \sim d_{\theta}(\rvx | \rvz^{\star})$.
\end{algorithmic}
\end{algorithm}

\subsection{Optimizing Designs With Cliqueformer}
%%SL.9.27: there are already details about design optimization that are sprinkled in the rest of the paper, can we collect all those details in one place (in this subsection)?

Once Cliqueformer is trained, we use it to optimize new designs.
Typically, MBO methods initialize this step at a sample of designs $(\rvx^{i_b})_{b=1}^{B}$ drawn from the dataset \citep{trabucco2021conservative}. 
Since in our algorithm the optimization takes place in the latent space $\mathcal{Z}$, we perform this step by encoding the sample of designs
with Cliqueformer's encoder, $\rvz^{i_b} \sim e_{\theta}(\rvz | \rvx^{i_b})$.
We then optimize the representation $\rvz^{i_b}$ of design $\rvx^{i_b}$ to maximize our model's value, 
\begin{align}
    \label{eq:mbo-objective}
    L_{\text{mbo}}\big( (\rvz^{i_b})_{b=1}^{B} \big) = \frac{1}{B} \sum_{b=1}^{B} f_{\theta}(\rvz^{i_b}),
\end{align}
at the same time minding the denominator of the regret bound from Theorem \ref{th:regret}.
That is, we don't want the optimizer to explore regions under which the marginal densities
$e_{\theta}(\rvz_C) = \E_{\rvx\sim \mathcal{D}}[e_{\theta}(\rvz_{C}|\rvx)]$ are small.
Fortunately, since the encoder was trained with standard-normal prior on the cliques, $p(\rvz_{C}) = N(0_{d_{\text{clique}}}, I_{d_{\text{clique}}})$, we know that values of $\rvz$ closer to the origin have unilaterally higher marginals.
This simple property of standard-normal distribution allows us to confine the optimizer's exploration to designs with in-distribution cliques by exponentially decaying the design at every optimization step.
Thus, we use AdamW as our optimizer \citep{loshchilov2017fixing} (see Appendix C for an ablation study).
We provide the pseudocode of the whole procedure of designign with Cliqueformer in Algorithm \ref{algorithm:bbo-cliqueformer}.

\section{Related Work}
The idea of using machine learning models in optimization problems has existed for a long time, and has been mainly cultivated in the literature on Bayesian optimization \citep[BO]{williams2006gaussian, snoek2012practical}.
The BO paradigm relies on two core assumptions: availability of data of examples paired with their target function values, as well as access to an oracle that allows a learning algorithm to query values of proposed examples.
Thus, similarly to reinforcement learning, the challenge of BO is to balance exploitation and exploration of the black-box function modeled by a Gaussian process.

Recently, to help BO tackle very high-dimensional problems, techniques of decomposing the target function have become more popular \citep{kandasamy2015high, rolland2018high}.
Most commonly, these methods decompose the target function into functions defined on the input's partitions.
While such models are likely to deviate far from the functions' ground-truth structure, one can derive theoretical guarantees for a range of such decompositions under the BO's query budget assumptions \citep{ziomek2023random}.
However, these results do not apply to our setting of offline MBO, where no additional queries are available and the immediate reliability on the model is essential.
Furthermore, instead of partitioning the input, we decompose our prediction over a latent variable that is learned by a transformer, enabling the model to acquire an expressive structure over which the decomposition is valid.

\begin{table*}[!t]
    \centering
    %\small
    \setlength{\tabcolsep}{6pt}
    \begin{tabular}{l c c c c c c c c}
        \textbf{Task} & \textbf{Grad.Asc.} & \textbf{RWR} & \textbf{IOM} & \textbf{COMs} & \textbf{DDOM} & \textbf{MatchOpt} & \textbf{Transformer} & \textbf{Cliqueformer} \\ 
        \hline
        Lat. RBF 11 & 0.00 & 0.08 & 0.00 & \textbf{0.66} & 0.00 & 0.56 & 0.47 & 0.65 \\ 
        Lat. RBF 31 & 0.00 & 0.31 & 0.00 & 0.50 & 0.00 & \textbf{0.66} & 0.00 & 0.64 \\ 
        Lat. RBF 41 & 0.00 & 0.35 & 0.00 & 0.45 & 0.00 & 0.32 & 0.20 & \textbf{0.66} \\ 
        Lat. RBF 61 & 0.00 & 0.29 & 0.00 & 0.25 & 0.00 & \textbf{0.74} & 0.16 & 0.66 \\ 
        Superconductor & 1.13 & 1.03 & 1.03 & 0.97 & 1.22 & 0.84 & 0.96 & \textbf{1.43} \\ 
        TF-Bind-8 & 0.99 & \textbf{1.58} & 0.99 & 1.57 & 1.55 & 1.47 & 1.48 & \textbf{1.58} \\ 
        DNA HEPG2 & \textbf{2.16} & 1.91 & 1.97 & 1.20 & 1.82 & 1.72 & 2.13 & 2.10 \\ 
        DNA k562 & 2.11 & 1.91 & 2.62 & 1.80 & 2.61 & 2.25 & 2.60 & \textbf{3.15} \\ 
        \hline 
        Ave.score $\uparrow$ & 0.80 & 0.93 & 0.83 & 0.93 & 0.90 & 1.07 & 1.00 & \textbf{1.36} \\ 
        Ave.rank $\downarrow$ & 5.00 & 4.00 & 5.00 & 4.38 & 4.75 & 4.38 & 4.63 & \textbf{1.63} \\ 
    \end{tabular}
    \caption{{ Comparison of Cliqueformer and the baselines. Each score is the empirical estimate of the expected top 1\% of scores, estimated by averaging the top-10 values out of 1000 generated designs, averaged over 5 runs. Scores are min-max normalized using the min and max from the dataset, ensuring that dataset designs fall in $[0, 1]$. Unlike \cite{trabucco2022design}, we do not use test data for normalization, improving interpretability (see Appendix D). We report both average score (higher is better) and average rank (lower is better).}}
    \label{tab:results}
\end{table*}

Offline model-based optimization (MBO) has gained traction in domains where BO assumptions cannot be met, offering solution generation from static datasets without additional queries. 
Early applications in molecule design \citep{gomez2018automatic} utilized variational auto-encoders \citep[VAE]{kingma2013auto} to learn continuous molecular representations. 
While this demonstrated deep learning's potential, it didn't explore the target function's structural properties. 
The Conditioning by Adaptive Sampling (CBaS) algorithm \citep{brookes2019conditioning} introduced a data type-agnostic approach for refining designs using non-differentiable oracle predictions. 
While this can be combined with trainable models through \emph{auto-focusing} \citep{fannjiang2020autofocused}, our setting differs by assuming a differentiable neural network model of the black-box function, enabling optimization through automatic differentiation \citep{paszke2019pytorch}.
\citet{trabucco2021conservative} introduced a neural network-based method, exactly for our setting, dubbed Conservative Objective Models (COMs), where a surrogate model is trained to both predict values of examples that can be found in the dataset, and penalize those that are not.
COMs differs from our work fundamentally, since its contribution lies in the formulation of the conservative regularizer applied to arbitrary neural networks, while we focus on scalable model architectures that facilitate computational design.
It is worth noting, however, that our Algorithm \ref{algorithm:bbo-cliqueformer}, similarly to COMs’s conservative regularizer, does constrain exploration of the design space, but it does so implicitly through weight decay.
A bit more similarly to us, Invariant Objective Model (IOM) of \citet{qi2022data} also trains a model with a latent representation. 
Their focus, however, is to make the representation distribution-invariant via GANs \citep{goodfellow2014generative}, while we focus on our network's decomposition abilities.

Another recent line of work proposes to tackle the design problem through means of generative modeling.
BONET \citep{mashkaria2023generative}, DDOM \citep{pmlr-v202-krishnamoorthy23a} are examples of works that bring the most recent novelties of the field to address design tasks.
BONET does so by training a transformer to generate sequences of designs that monotonically improve in their value,
DDOM by training a value-conditioned diffusion model.
That is, these methods attempt to generate high-value designs through novel conditional generation mechanisms.
Instead, we model MBO as a maximization problem, and propose a scalable model that acquires the structure of the black-box function through standard gradient-based learning.
To this end, we bring powerful techniques from generative modeling, like transformers \citep{vaswani2017attention} and variational-information bottlenecks \citep{kingma2013auto, alemi2016deep}.

The work of \citet{grudzien2024functional} introduced the theoretical foundations of functional graphical models (FGMs), including Theorem \ref{th:regret}.
However, as we have shown in Theorem \ref{lemma:rotation}, their graph discovery heuristic does not solve the problem of learning the black-box function's structure.
On the other hand, we solved it by its integration of a pre-defined FGM into the architecture of our \emph{Cliqueformer}.
As such, the model learns the structure of the target function, as well as learns to predict its value, in synergy within an end-to-end training.
%For more related works, please refer to Appendix \ref{ap:rel}.

\section{Experiments}
\label{sec:exp}

In this section, we provide the empirical evaluation of Cliqueformer.
We begin by benchmarking Cliqueformer against prior methods on tasks from the MBO literature.
We then evaluate the effect of the novel FGM decomposition layer in Cliqueformer through an ablation study.
\begin{figure*}[t]
  \centering

  \begin{minipage}[t]{0.235\textwidth}
    \centering
    \includegraphics[width=\linewidth]{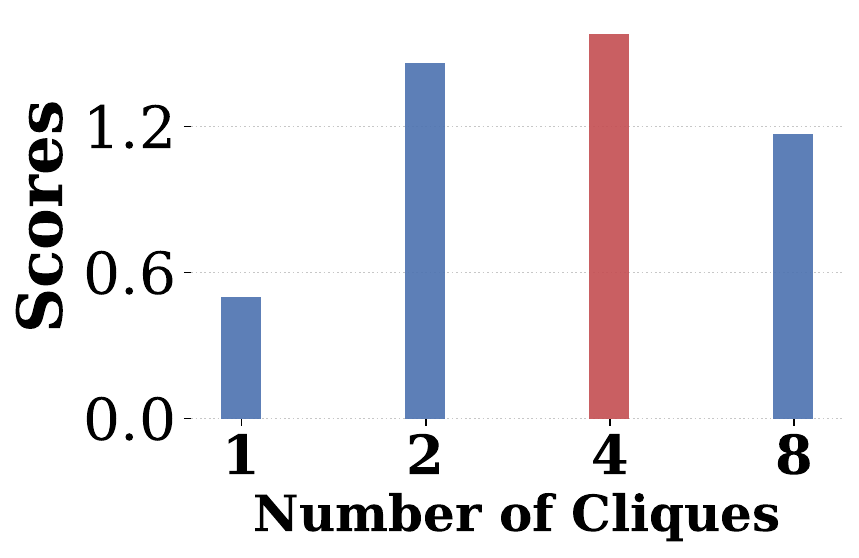}\\[-2pt]
    (a) TFBind-8 (\(d_{z}\) fixed)
  \end{minipage}\hfill
  \begin{minipage}[t]{0.235\textwidth}
    \centering
    \includegraphics[width=\linewidth]{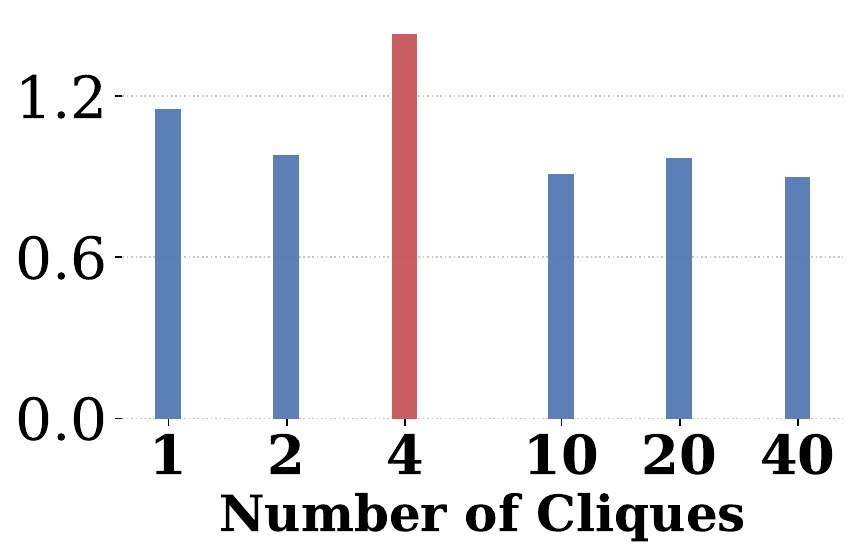}\\[-2pt]
    (b) Supercond. (\(d_{z}\) fixed)
  \end{minipage}\hfill
  \begin{minipage}[t]{0.235\textwidth}
    \centering
    \includegraphics[width=\linewidth]{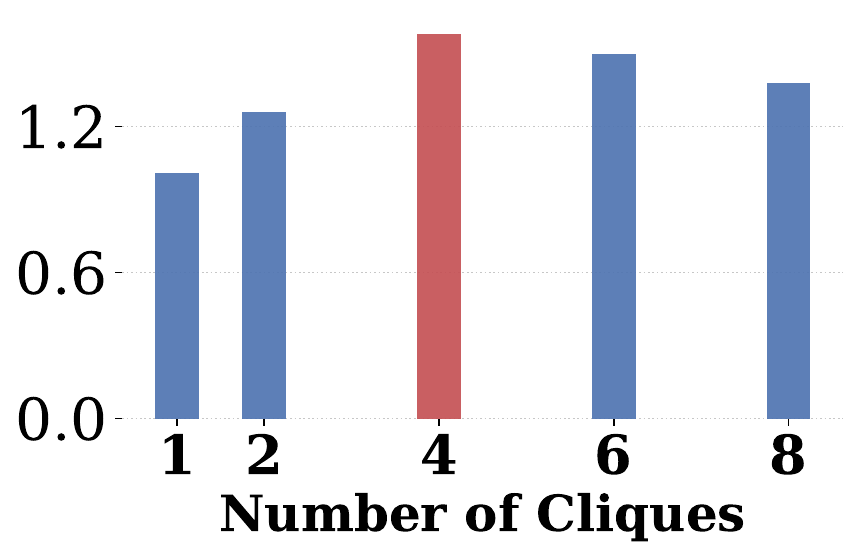}\\[-2pt]
    (c) TFBind-8 (\(d_{\text{clique}}\) fixed)
  \end{minipage}\hfill
  \begin{minipage}[t]{0.235\textwidth}
    \centering
    \includegraphics[width=\linewidth]{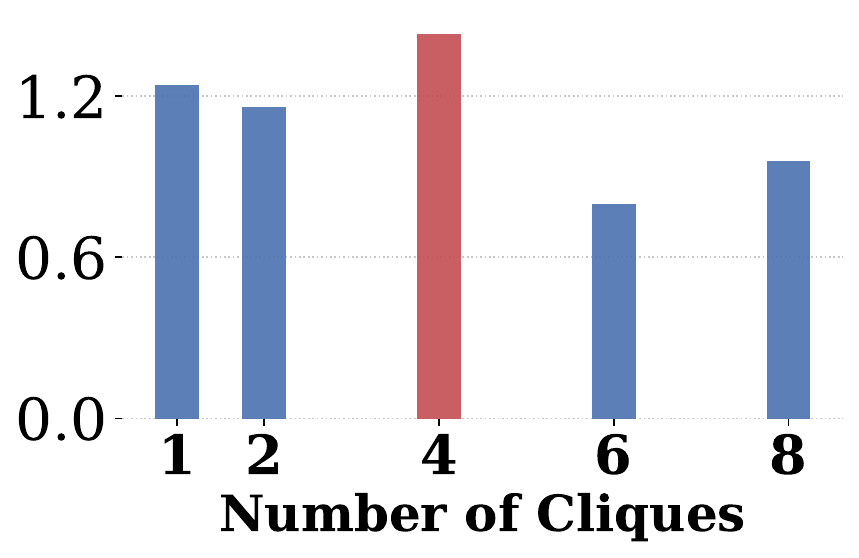}\\[-2pt]
    (d) Supercond. (\(d_{\text{clique}}\) fixed)
  \end{minipage}
    \caption{{Ablation experiments on the number of cliques $N_{\text{clique}}$ used in the FGM decomposition of Cliqueformer (the $y$-axis shows the design score and the $x$-axis shows the number of cliques used).
    In Figures (a) \& (b), for each task, we fixed the size $d_{z}$ of the latent variable $\rvz$ and varied the number of cliques $N_{\text{clique}}$. 
    In Figures (c) \& (d), for each task, we fixed the clique size $d_{\text{clique}}$, and varied the number of cliques $N_{\text{clique}}$.
    In every figure, the red bar corresponds to the best (maximal) score.}}
    \label{fig:abl}
\end{figure*}

\subsection{Benchmarking}
We compare our model to five classes of algorithms, each represented by proven prior methods.
As a \emph{na\"ive} baseline, we employ gradient ascent on a learned model (\emph{Grad. Asc.}).
%%SL.1.30: could we add a few words about why we believe this is a strong baseline? eg reference some prior papers and say this was surprisingly strong in prior work?
To compare to \emph{exploratory} methods, we use Reward-Weighted Regression \citep[\emph{RWR}]{peters2007reinforcement} which learns by regressing the policy against its most promising perturbations.
To represent \emph{conservative} algorithms, we compare to Conservative Objective Models \citep[\emph{COMs}]{trabucco2021conservative, kumar2021data} and Invariant Objective Model \citep[\emph{IOM}]{qi2022data}.
As a \emph{conditional generative modeling} in MBO baseline, we use Denoising Diffusion Optimization Models \citep[DDOM]{pmlr-v202-krishnamoorthy23a}, as well as recent MatchOpt \citep{hoang2025learning}, which are inspired by generative diffusion models \citep{ho2020denoising}.
Additionally, to evaluate the importance of the novel elements of our architecture and training, we evaluate gradient ascent with the transformer backbone \citep[\emph{Transformer}]{vaswani2017attention}.  
We use standard hyperparameters for the baselines.
While Cliqueformer can be tuned for each task, we keep most of the hyperparameters the same---see Appendix D for more details.

Following prior work \citep{trabucco2022design}, we normalize scores to $[0, 1]$ based on the dataset's range and report the empirical top 1\%, estimated by averaging the top 10 designs from 1000 candidates across 5 seeds. Below, we describe benchmarks that we use for our experiments.

\paragraph{Latent Radial-Basis Functions (Lat. RBF).}
These synthetic tasks, introduced by \citet{grudzien2024functional}, stress-test MBO methods. Each datapoint is generated by sampling $\rvz \sim \mathcal{N}(0, I_{d_z})$ for $d_z \in {11, 31, 41, 61}$, applying RBFs over $d_C$-dimensional cliques of a fixed FGM to get $\ry$, and mapping $\rvz$ through a nonlinear transformation $\rvx = T(\rvz) \in \mathcal{T} \subset \mathbb{R}^d$ with $d > d_z$. Only $\rvx$ and $\ry$ are observed. Validity of a new design $\hat{\rvx}$ is judged by whether $\hat{\rvx}$ lies on the manifold $\mathcal{T}$, i.e., whether $T^{-1}(\hat{\rvx})$ exists. 
In our experiments, invalid designs score 0. 
These tasks also test a model’s ability to leverage RBF structure which becomes increasingly important as the task's dimension grows. 
As Table~\ref{tab:results} shows, Cliqueformer sustains strong performance as predicted by Theorem~\ref{th:regret}, whereas some baselines, like COMs, decline in higher dimensions.

\paragraph{Superconductor.}
This task involves designing an 81-dimensional material to maximize its critical temperature \citep{hamidieh2018data, fannjiang2020autofocused, trabucco2022design}, thus evaluating MBO methods in real-valued, continuous domains. Here, greedy optimization (e.g., gradient ascent) is highly effective. Though Cliqueformer is designed to be robust to distribution shift, its ability to “stitch” in-distribution cliques enables substantial gains. It outperforms all baselines, particularly COMs--—limited by its conservative updates—--and gradient ascent (w/ and wo/ Transformer), that lacks Cliqueformer's decomposition strategy.

%\vspace{-10pt}
\paragraph{TFBind-8 \& DNA Enhancers.}
These tasks optimize discrete DNA sequences of lengths 8 and 200, respectively \citep{trabucco2022design, uehara2024bridging}. TFBind-8 measures binding affinity for a transcription factor, while DNA Enhancers target HEPG2 and k562 expression levels. TFBind-8, though simple, demonstrates Cliqueformer's ability to solve discrete tasks. For DNA Enhancers—--large-scale tasks with $\sim 2\times10^5$ samples—--Cliqueformer matches the strong performance of gradient ascent on HEPG2 and significantly outperforms all baselines on k562, showing its scalability. Averaged across all tasks, Cliqueformer achieves the highest overall performance.

\subsection{Ablations}
%%SL.9.28: I think there is only one ablation, right, so it should be "ablation" (singular) rather than plural? Besides that, it's probably more traditional to have this come last -- the logic is that readers don't really care what the ablations are if they are not yet convinced that the method works well, and it's the benchmarks that will convince them it works well. Also, the Fig 5 ablations don't make sense if the tasks have not yet been introduced, and the tasks are only introduced in Sec 5.2.

%While the decomposing mechanism of Cliqueformer is entirely novel, other components, such as transformer blocks \citep{vaswani2017attention} and variational-information bottlenecks \citep{kingma2013auto, alemi2016deep} are known and powerful deep learning tools.

While some components of Cliqueformer, such as transformer blocks \citep{vaswani2017attention} and variational-information bottlenecks \citep{kingma2013auto, alemi2016deep}, are well-known deep learning tools, the decomposing mechanism of our model is novel.
In Figure (\ref{fig:abl}), we verify the utility of the decomposing mechanism with an ablation study, in which we sweep over the number of cliques $N_{\text{clique}}$ of Cliqueformer in the discrete TFBind-8 and continuous Superconductor tasks. 
In each of the tasks, we fix the size $d_{z}$ of the latent variable $\rvz$ so that it approximately matches the dimensionality of the data, and sweep over the number of cliques into which it can be decomposed under the formula $d_{z}=d_{\text{knot}} + N_{\text{clique}} (d_{\text{clique}} - d_{\text{knot}})$.
We cover the case $N_{\text{clique}}=1$ to compare Cliqueformer to an FGM-oblivious VAE with transformer backbone.% and AdamW design optimizer.
%%SL.9.28: It might be quite unclear to readers that the case of 1 clique is actually the "prior work" baseline -- can you explain the siginficance of 1 vs > 1?

Results in Figure (\ref{fig:abl}) demonstrate the essential role of FGM decomposition in achieving superior performance across both tasks. The results reveal a clear trade-off in parameter selection: with fixed $d_{z}$, too few cliques (particularly $N_{\text{clique}}=1$) prevent effective function decomposition, while too many cliques oversimplify the target function.
We further investigated the impact of varying the number of cliques while maintaining fixed clique sizes ($d_{\text{clique}}=3$ for TFBind-8 and $d_{\text{clique}}=21$ for Superconductor). Figure (\ref{fig:abl}) also confirms that the optimal $N_{\text{clique}}$ values remain unchanged, likely because deviating from these values results in latent spaces that are either too constrained or too expansive relative to the data dimensionality.

In our experiments, we consistently obtained good performance by setting the clique size to $d_{\text{clique}}=3$ and choosing the number of cliques so that the total latent dimension approximately matches that of the design.
In DNA Enhancers, we doubled the clique size and halved the number of cliques to decrease the computational cost of attention layers.
We offer more ablation studies in Appendix C,
and we include a detailed report of hyperparameters in Appendix D.
An in-depth analysis of the relation between hyperparameters and the performance is an exciting avenue of future work.

\section{Conclusion}
\label{sec:conc}
In this work, we introduced Cliqueformer, a scalable neural architecture for offline model-based optimization that leverages functional graphical models to learn the structure of black-box target functions. By incorporating recent theoretical advances in MBO, our model eliminates the need for explicit conservative regularization or iterative retraining when proposing designs. Cliqueformer achieves and maintains strong performance across all evaluation tasks, paving the way for research into scaling MBO with large neural networks and expanded design datasets.

\bibliography{aaai2026}

\clearpage

\appendix
\onecolumn

\section{More Related Work}
\label{ap:rel}
Several reinforcement learning (RL) approaches have been explored extensively for biological sequence design. DyNA-PPO \citep{angermueller2019model} leverages proximal policy optimization \citep{schulman2017proximal} with a model-based variant to improve sample efficiency in the low-round setting typical of wet lab experiments. PEX, also resembling the PPO \citep{schulman2017proximal} learning style, \citep{ren2022proximal} prioritizes local search through directed evolution \citep{arnold1998design} while using a specialized architecture for modeling fitness landscapes. FBGAN \citep{gupta2018feedback} introduces a feedback loop mechanism to optimize synthetic gene sequences using an external analyzer. However, these methods fundamentally rely on active learning and iterative refinement through oracle queries - DyNA-PPO requires simulator fitting on new measurements, PEX conducts proximal exploration, and FBGAN uses feedback loops with an external analyzer. This makes them unsuitable for offline MBO settings where no additional queries are allowed. Furthermore, while these approaches are specialized for biological sequences, offline MBO aims to tackle a broader class of design problems through static dataset learning.

\section{Theoretical details}
\label{ap:the}
The full statement of the following theorem considers an MBO algorithm with function clas $\mathcal{F}=\mathcal{F}_{C_1} \oplus ... \oplus \mathcal{F}_{C_{N_{\text{clique}}}}$, so that every of its element has form 
\begin{center}
    $f(\rvx) = \sum_{i=1}^{N_{\text{clique}}} f_{C_i}(\rvx_{C_i})$.
\end{center}
As described in Section \ref{sec:cliqueformer}, Cliqueformer's architecture forms such a function class on top of the learned latent space.
We define the statistical constant as 
\begin{center}
    $C_{\text{stat}} = \sqrt{\frac{1}{1-\sigma}}$, \quad 
    where 
    $\sigma = \max\limits_{C_i \neq C_j, \hat{f}_{C_i}, \hat{f}_{C_j}} \mathbb{C}orr_{\rvx \sim p}[ \hat{f}_{C_i}(\rvx_{C_i}), \hat{f}_{C_j} ]$\\
\end{center}
and the function approximation complexity constant as 
\begin{center}
    $C_{\text{cpx}} = \sqrt{ \frac{ N_{\text{clique}} \sum_{i=1}^{N_{\text{clique}}} \log(|\mathcal{F}_{C_i}|/\delta)  }{ N } }$,
\end{center}
where $\delta$ is the PAC error probability \citep{shalev2014understanding}.

\regret*

\rotation*
\begin{proof}
Since the density of $\rvx$ is positive and continuous, we can form a bijection that maps $\rvx$ to another random variable $\rvz \in \mathbb{R}^{l}$, where $l \leq d$, that follows the standard-normal distribution \citep[Appendix E]{dai2019diagnosing}. 
We denote this bijection as $Z(\rvx)$.
Let us define
\begin{center}
   $\ry=f^{z}(\rvz) = \exp\big(\frac{1}{\sqrt{l}}\sum_{i=1}^{l} \rz_i \big)$.  
\end{center}
Then, the FGM of $f^{z}$ has an edge between every two variables since each variable's partial derivative 
\begin{center}
    $\frac{\partial f^{z}}{\partial \rz^{i}} = \frac{1}{\sqrt{l}}\exp\big(\frac{1}{\sqrt{l}}\sum_{i=1}^{d} \rz_i \big)$
\end{center}
is also a function of all others \citep[Lemma 1]{grudzien2024functional}. 
Consider now a rotation $\rho: \rvz \mapsto \rvv = (\rv_1, \dots, \rv_l)$ such that $\rv_1 = \frac{1}{\sqrt{l}}\sum_{i=1}^{l}\rz_i$.  
Then, $\rvv \sim N(0_l, I_l)$, and $\ry$ can be expressed in terms of $\rvv$ as $\ry=f^{v}(\rvv) = \exp(\rv_1)$.
Then, the FGM of $f^{v}$ has no edges, since it depends on only one variable, inducing no interactions between any two variables. 
Recall that $\rvx = Z^{-1}(\rvz)$.
Then, $\rvx$ be represented by standard-normal $\rvz$ and $\rvv$, obtainable by
\begin{center}
    $\rvz=Z(\rvx)$ and $\rvv = \rho(\rvz) = \rho\big( Z(\rvx) \big)$.
\end{center}

Furthermore, we can define
\begin{center}
    $f(\rvx) = f^{z}\big(Z(\rvx)\big)$
\end{center}
which is identically equal to $f^{z}(\rvz)$ and $f^{v}(\rvv)$, which have a complete and an empty FGM, respectively, thus fulfilling the theorem's claim.
\end{proof}

\paragraph{Computational complexity.} Below (Table \ref{tab:complexities}), we list the computational complexities, as the order of the number of FLOPs, for each method's training step and design optimization phase, as a function of batch size $B$, number of model layers $L$, model's hidden dimension $H$, number of exploratory perturbations $P$, number of adversarial training sub-steps $A$, number of design optimization steps $T$, and the number of cliques in an FGM-based model $C$. Note that the majority of quadratic terms, such as $H^2$ and $C^2$ do not influence runtime much if parallelized on a GPU/TPU. We print in bold terms, such as $\mathcal{\bold{T}}$, that contribute to the complexity with sequential operations, thus inevitably affecting the runtime.

\begin{table}[ht]
\centering
\begin{tabular}{lcc}
\toprule
\textbf{Method} & \textbf{Training Step Complexity} & \textbf{Design Complexity} \\
\midrule
Grad Asc. & $\mathcal{O}(BLH^2)$ & $\mathcal{O}(\mathbf{T}LH^2)$ \\
RWR & $\mathcal{O}(BLH^2)$ & $\mathcal{O}(\mathbf{T}PLH^2)$ \\
COMs & $\mathcal{O}((\mathbf{A}+B)LH^2)$ & $\mathcal{O}(\mathbf{T}LH^2)$ \\
DDOM & $\mathcal{O}(BLH^2)$ & $\mathcal{O}(\mathbf{T}LH^2)$ \\
IOM & $\mathcal{O}(BLH^2)$ & $\mathcal{O}(\mathbf{T}LH^2)$ \\
Match-Opt & $\mathcal{O}(B\mathbf{P}LH^2)$ & $\mathcal{O}(\mathbf{T}LH^2)$ \\
Transformer & $\mathcal{O}(BLHD(H+D))$ & $\mathcal{O}(\mathbf{T}LHD(H+D))$ \\
Cliqueformer & $\mathcal{O}(BLH(D(H+D) + C(H+C)))$ & $\mathcal{O}(\mathbf{T}LH(D(H+D) + C(H+C)))$ \\
\bottomrule
\end{tabular}
\caption{Computational complexities (in terms of FLOPs) of methods from Section \ref{sec:exp}.}
\label{tab:complexities}
\end{table}

\newpage

\section{Extra Experiments}
\label{app:extra-ablations}

\paragraph{Complete results.} In this section, because of the format constraints on the main paper, we re-report the results from Table \ref{tab:results}, additionally including the standard deviation across seeds, in Table \ref{tab:all-results}.
\begin{table*}[!t]
    \centering
    %\small
    \setlength{\tabcolsep}{5pt}
    \begin{tabular}{l c c c c c c c c}
        \textbf{Task} & \textbf{Grad.Asc.} & \textbf{RWR} & \textbf{IOM} & \textbf{COMs} & \textbf{DDOM} & \textbf{MatchOpt} & \textbf{Transformer} & \textbf{Cliqueformer} \\ 
        \hline
        Lat. RBF 11 & 0.00 \tiny$\pm$ 0.00 & 0.08 \tiny$\pm$ 0.08 & 0.00 \tiny$\pm$ 0.00 & \textbf{0.66} \tiny$\pm$ 0.04 & 0.00 \tiny$\pm$ 0.00 & 0.56 \tiny$\pm$ 0.01 & 0.47 \tiny$\pm$ 0.05 & 0.65 \tiny$\pm$ 0.07 \\ 
        Lat. RBF 31 & 0.00 \tiny$\pm$ 0.00 & 0.31 \tiny$\pm$ 0.10 & 0.00 \tiny$\pm$ 0.00 & 0.50 \tiny$\pm$ 0.05 & 0.00 \tiny$\pm$ 0.00 & \textbf{0.66} \tiny$\pm$ 0.01 & 0.00 \tiny$\pm$ 0.00 & 0.64 \tiny$\pm$ 0.05 \\ 
        Lat. RBF 41 & 0.00 \tiny$\pm$ 0.00 & 0.35 \tiny$\pm$ 0.08 & 0.00 \tiny$\pm$ 0.00 & 0.45 \tiny$\pm$ 0.06 & 0.00 \tiny$\pm$ 0.00 & 0.32 \tiny$\pm$ 0.05 & 0.20 \tiny$\pm$ 0.01 & \textbf{0.66} \tiny$\pm$ 0.05 \\ 
        Lat. RBF 61 & 0.00 \tiny$\pm$ 0.00 & 0.29 \tiny$\pm$ 0.10 & 0.00 \tiny$\pm$ 0.00 & 0.25 \tiny$\pm$ 0.04 & 0.00 \tiny$\pm$ 0.00 & \textbf{0.74} \tiny$\pm$ 0.02 & 0.16 \tiny$\pm$ 0.03 & 0.66 \tiny$\pm$ 0.05 \\ 
        Superconductor & 1.13 \tiny$\pm$ 0.08 & 1.03 \tiny$\pm$ 0.07 & 1.03 \tiny$\pm$ 0.06 & 0.97 \tiny$\pm$ 0.08 & 1.22 \tiny$\pm$ 0.08 & 0.84 \tiny$\pm$ 0.05 & 0.96 \tiny$\pm$ 0.05 & \textbf{1.43} \tiny$\pm$ 0.04 \\ 
        TF-Bind-8 & 0.99 \tiny$\pm$ 0.00 & \textbf{1.58} \tiny$\pm$ 0.03 & 0.99 \tiny$\pm$ 0.00 & 1.57 \tiny$\pm$ 0.02 & 1.55 \tiny$\pm$ 0.03 & 1.47 \tiny$\pm$ 0.08 & 1.48 \tiny$\pm$ 0.03 & \textbf{1.58} \tiny$\pm$ 0.01 \\ 
        DNA HEPG2 & \textbf{2.16} \tiny$\pm$ 0.07 & 1.91 \tiny$\pm$ 0.12 & 1.97 \tiny$\pm$ 0.12 & 1.20 \tiny$\pm$ 0.09 & 1.82 \tiny$\pm$ 0.10 & 1.72 \tiny$\pm$ 0.11 & 2.13 \tiny$\pm$ 0.06 & 2.10 \tiny$\pm$ 0.07 \\ 
        DNA k562 & 2.11 \tiny$\pm$ 0.06 & 1.91 \tiny$\pm$ 0.11 & 2.62 \tiny$\pm$ 0.20 & 1.80 \tiny$\pm$ 0.12 & 2.61 \tiny$\pm$ 0.21 & 2.25 \tiny$\pm$ 0.18 & 2.60 \tiny$\pm$ 0.19 & \textbf{3.15} \tiny$\pm$ 0.07 \\ 
        \hline 
        Ave.score $\uparrow$ & 0.80 & 0.93 & 0.83 & 0.93 & 0.90 & 1.07 & 1.00 & \textbf{1.36} \\ 
        Ave.rank $\downarrow$ & 5.00 & 4.00 & 5.00 & 4.38 & 4.75 & 4.38 & 4.63 & \textbf{1.63} \\ 
    \end{tabular}
    \caption{{ Comparison of Cliqueformer and the baselines. Each score is the empirical estimate of the expected top 1\% of scores, estimated by averaging the top-10 values out of 1000 generated designs, averaged over 5 runs. Scores are min-max normalized using the min and max from the dataset, ensuring that dataset designs fall in $[0, 1]$. Unlike \cite{trabucco2022design}, we do not use test data for normalization, improving interpretability (see Appendix~\ref{ap:exp}). Standard deviations are computed for each of the top-10 samples and averaged across runs. We report both average score (higher is better) and average rank (lower is better).}}
    \label{tab:all-results}
\end{table*}

\paragraph{Ablations.} We ablate the components of our model, and its design optimization algorithm, that served as regularizers 
and were inspired theoretically.
That is, we examine the VIB term in the loss from Equation (\ref{eq:loss}), and the weight decay step in Algorithm \ref{algorithm:bbo-cliqueformer}. 
We have not run the ablations for DNA Enhancers tasks due to their high computational cost - around 11 hours on a Google TPU v3-8 for a single run.
The results are reported in Table \ref{table:ablation}.

\begin{table*}[!htbp]
\centering
\setlength{\tabcolsep}{4pt}
\resizebox{\textwidth}{!}{
\begin{tabular}{l|c c c c c c|c c}
%\hline
Version \textbackslash Task  & Lat. RBF 11 & Lat. RBF 31 & Lat. RBF 41 & Lat. RBF 61 & Superconductor & TF-Bind-8 & Ave.Score $\uparrow$ & Ave.rank $\downarrow$\\
\hline
Base & 0.65 \tiny$\pm$ 0.07 & 0.64 \tiny$\pm$ 0.05 & 0.66 \tiny$\pm$ 0.05  & \textbf{0.66} \tiny$\pm$ 0.05  & 1.43 \tiny$\pm$ 0.04 & \textbf{1.58} \tiny$\pm$ 0.01 & \textbf{0.94} & \textbf{1.83}\\
\hline
No VIB & 0.64 \tiny$\pm$ 0.06 & 0.68 \tiny$\pm$ 0.05 & \textbf{0.67} \tiny$\pm$ 0.04  & \textbf{0.66} \tiny$\pm$ 0.05  & 0.92 \tiny$\pm$ 0.02 & 1.47 \tiny$\pm$ 0.06 & 0.84 & 2.17\\
%\hline
No Weight Decay & \textbf{0.66} \tiny$\pm$ 0.06 & \textbf{0.69} \tiny$\pm$ 0.06 & 0.64 \tiny$\pm$ 0.07  & 0.63 \tiny$\pm$ 0.05  & \textbf{1.45} \tiny$\pm$ 0.04 & 1.48 \tiny$\pm$ 0.06 & 0.93 & 2.00\\
%\hline
\end{tabular}
}
\caption{Comparison of Cliqueformer (the \emph{Base} version) to ablative versions---one without the VIB term (\emph{No VIB}) in Equation \ref{eq:loss}, and one where the model was trained just like the base version but the design optimizer does not use weight decay (\emph{No Weight Decay}). }
\label{table:ablation}
\end{table*}
We observe that taking out the VIB term can significantly harm the performance, as showcased by the Superconductor task. 
If the term is kept but the weight decay step is removed from the design optimizer, the final results slightly underperform the full model and algorithm. 
Thus, we recommend to keep these regularizers.

\paragraph{Different metric.} In the main body of the paper, the results in Table \ref{tab:results} are our estimates of the expected top 1\% of the distribution of the design values, computed by taking the average of top 10 values from a batch of 1000 designs, averaged over 5 random seeds. Our rationale behind it was to establish a more concrete metric of evaluation, like the expected top 1\%, instead of reporting the maximum of 128 designs that is more common in recent literature \citep{trabucco2022design, pmlr-v202-krishnamoorthy23a}. In fact, we argue that top 1 of 128 is an estimate of the top 1\% too (since 1/128=0.78\%), but one with more bias and variance. 
However, for completeness, we report results with this metric (top 1 of 128, averaged over random seeds) in Table \ref{tab:results-128}.

\begin{table*}[h!]
    \centering
    \setlength{\tabcolsep}{6pt}
    \begin{tabular}{l c c c c c c}
        Task & Grad.Asc. & RWR & COMs & DDOM & Transformer & Cliqueformer\\ 
        \hline
        Lat. RBF 11 & 0.00 & 0.00 & \textbf{0.74} & 0.00 & 0.48 & 0.72 \\ 
        Lat. RBF 31 & 0.00 & 0.00 & 0.49 & 0.00 & 0.00 & \textbf{0.65} \\ 
        Lat. RBF 41 & 0.00 & 0.00 & 0.41 & 0.00 & 0.00 & \textbf{0.66} \\ 
        Lat. RBF 61 & 0.00 & 0.00 & 0.00 & 0.00 & 0.00 & \textbf{0.62} \\ 
        Superconductor & 1.19 & 1.00 & 0.92 & 1.20 & 0.91 & \textbf{1.42} \\ 
        TF-Bind-8 & 0.99 & \textbf{1.60} & 1.56 & 1.55 & 1.52 & 1.48 \\ 
        \hline 
        Ave.score $\uparrow$ & 0.36 & 0.43 & 0.70 & 0.46 & 0.49 & \textbf{1.09} \\ 
        Ave.rank $\downarrow$ & 3.33 & 2.83 & 2.33 & 3.00 & 3.50 & \textbf{1.83} \\ 
    \end{tabular}
    \caption{ {\small Experimental results of Cliqueformer and the baselines.
    Each score is the maximum of values of 128 designs, averaged over 5 runs.
    }}
    \label{tab:results-128}
\end{table*}

For comparison, we bring back results reported in Section \ref{sec:exp} in Table \ref{tab:results-10-1000}.

\begin{table*}[h!]
    \centering
    \setlength{\tabcolsep}{6pt}
    \begin{tabular}{l c c c c c c}
        Task & Grad.Asc. & RWR & COMs & DDOM & Transformer & Cliqueformer\\ 
        \hline
        Lat. RBF 11 & 0.00 & 0.08 & \textbf{0.66} & 0.00 & 0.47 & 0.65 \\ 
        Lat. RBF 31 & 0.00 & 0.31 & 0.50 & 0.00 & 0.00 & \textbf{0.64} \\ 
        Lat. RBF 41 & 0.00 & 0.35 & 0.45 & 0.00 & 0.20 & \textbf{0.66} \\ 
        Lat. RBF 61 & 0.00 & 0.29 & 0.25 & 0.00 & 0.16 & \textbf{0.66} \\ 
        Superconductor & 1.13 & 1.03 & 0.97 & 1.22 & 0.96 & \textbf{1.43} \\ 
        TF-Bind-8 & 0.99 & \textbf{1.58} & 1.57 & 1.55 & 1.48 & \textbf{1.58} \\ 
        \hline 
        Ave.score $\uparrow$ & 0.35 & 0.61 & 0.73 & 0.46 & 0.55 & \textbf{1.10} \\ 
        Ave.rank $\downarrow$ & 4.67 & 2.83 & 2.67 & 4.17 & 4.33 & \textbf{1.17} \\ 
    \end{tabular}
    \caption{ {\small Experimental results of Cliqueformer and the baselines from Section \ref{sec:exp}.
    Each score is the empirical estimate of the top 1\% by averaging top 10 values out of 1000 samples, and averaging over 5 runs.
    }}
    \label{tab:results-10-1000}
\end{table*}

We observe very similar scores to those in Table \ref{tab:results}, but with the variance effect being visible: for example, RWR get poorer results in Lat. RBF 11 (0.00 as compared to 0.08), but stronger in TFBind8 (1.60 as compared to 1.58).

\paragraph{Lat. RBF Validity Scores.} We re-report, in Table \ref{tab:percent}, the evaluation results on Latent RBF tasks, this time additionally reporting the percentage of \emph{valid} designs---the proportions of designs $\rvx$, in the final sample, for which $T^{-1}(\rvx)$ is well-defined.
We find that, while half of the baselines tends to easily slip outside of the valid manifold, Cliqueformer steadily produces designs whose majority are valid. It is worth noting, however, that MathOpt attains far higher validity score, while also displaying larger variance of the design value scores. We suspect that this is due to their generative-style approach, wherein a (score-maximizing) gradient field is learned \cite{ho2020denoising}, which enables the designs to stay closer to the training distribution. Hence, higher average validity than Cliqueformer (74\% vs. 73\%), yet lower scores (0.57 vs. 0.65).  
\begin{table*}[!t]
    \centering
    \small
    \setlength{\tabcolsep}{6pt}
    \begin{tabular}{l c c c c c c c c}
        \textbf{Task} & \textbf{Grad.Asc.} & \textbf{RWR} & \textbf{IOM} & \textbf{COMs} & \textbf{DDOM} & \textbf{MatchOpt} & \textbf{Transformer} & \textbf{Cliqueformer} \\
        \hline
        Lat. RBF 11 & 0.00 (0\%) & 0.08 (1\%) & 0.00 (0\%) & \textbf{0.66} (72\%) & 0.00 (0\%) & 0.56 (98\%) & 0.47 (60\%) & 0.65 (74\%)\\ 
        Lat. RBF 31 & 0.00 (0\%) & 0.31 (3\%) & 0.00 (0\%) & 0.50 (32\%) & 0.00 (0\%) & \textbf{0.66} (92\%) & 0.00 (0\%) & \textbf{0.64} (76\%)\\ 
        Lat. RBF 41 & 0.00 (0\%) & 0.35 (3\%) & 0.00 (0\%) & 0.45 (16\%) & 0.00 (0\%) & 0.32 (58\%) & 0.20 (75\%) & \textbf{0.66} (72\%)\\ 
        Lat. RBF 61 & 0.00 (0\%) & 0.29 (4\%) & 0.00 (0\%) & 0.25 (7\%) & 0.00 (0\%) & \textbf{0.74} (98\%) & 0.16 (64\%) & 0.66 (68\%)\\ 
    \end{tabular}
    \caption{ {\small Experimental results of Cliqueformer and the baselines. We report the average score of top-10 of 1000 designs and their validity scores---the fraction of valid designs in the sample of 1000, averaged over 5 seeds.}
    }
    \label{tab:percent}
\end{table*}

\begin{comment}
 \begin{table*}
\centering
\setlength{\tabcolsep}{4pt}
\resizebox{0.7\textwidth}{!}{
\begin{tabular}{l||c|c|c|c}
\hline
\textbf{Version \textbackslash Task } & \textbf{Lat. RBF 11} & \textbf{Lat. RBF 31} & \textbf{Lat. RBF 41} & \textbf{Lat. RBF 61}  \\
\hline
Base & 0.65 \tiny$\pm$ 0.07 & 0.64 \tiny$\pm$ 0.05 & 0.66 \tiny$\pm$ 0.05  & 0.66 \tiny$\pm$ 0.05  \\
\hline
No VIB & 0.64 \tiny$\pm$ 0.06 & 0.68 \tiny$\pm$ 0.05 & 0.67 \tiny$\pm$ 0.04  & 0.66 \tiny$\pm$ 0.05  \\
\hline
No Weight Decay & 0.66 \tiny$\pm$ 0.06 & 0.69 \tiny$\pm$ 0.06 & 0.64 \tiny$\pm$ 0.07  & 0.63 \tiny$\pm$ 0.05 \\
\hline
\end{tabular}
}
\caption{Comparison of Cliqueformer (the \emph{Base} version) to ablative versions---one without the VIB term (\emph{No VIB}) in Equation \ref{eq:loss}, and one where design optimizer does not use weight decay (\emph{No Weight Decay}). }
\label{table:ablation}
\end{table*}   
\end{comment}

\begin{figure*}[t]
  \centering

  \begin{minipage}[t]{0.31\textwidth}
    \centering
    \includegraphics[width=\linewidth]{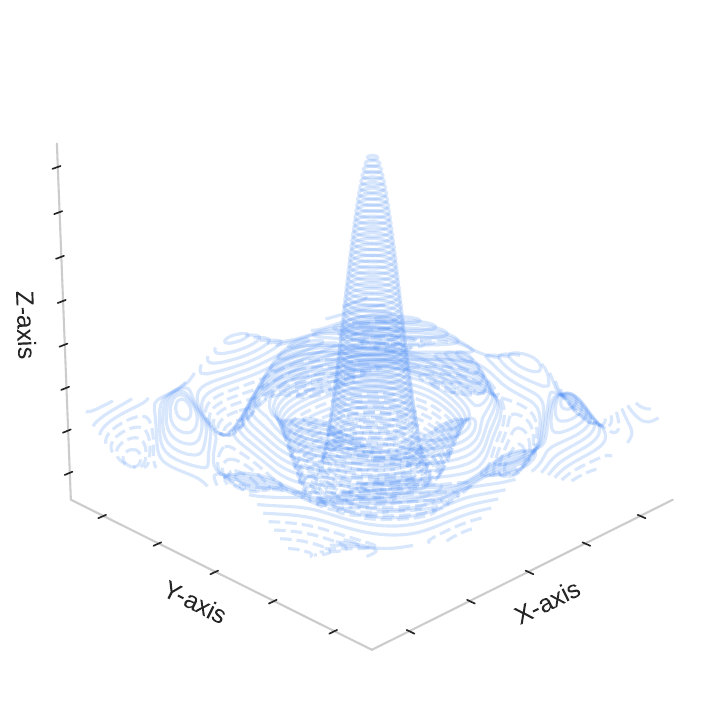}\\[-2pt]
    \footnotesize (a) Contours in XY-planes
  \end{minipage}\hfill
  \begin{minipage}[t]{0.31\textwidth}
    \centering
    \includegraphics[width=\linewidth]{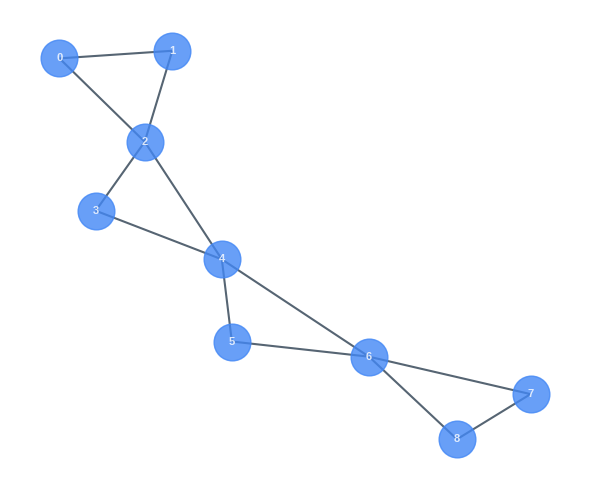}\\[-2pt]
    \footnotesize (b) Chain-of-triangles FGM
  \end{minipage}\hfill
  \begin{minipage}[t]{0.31\textwidth}
    \centering
    \includegraphics[width=\linewidth]{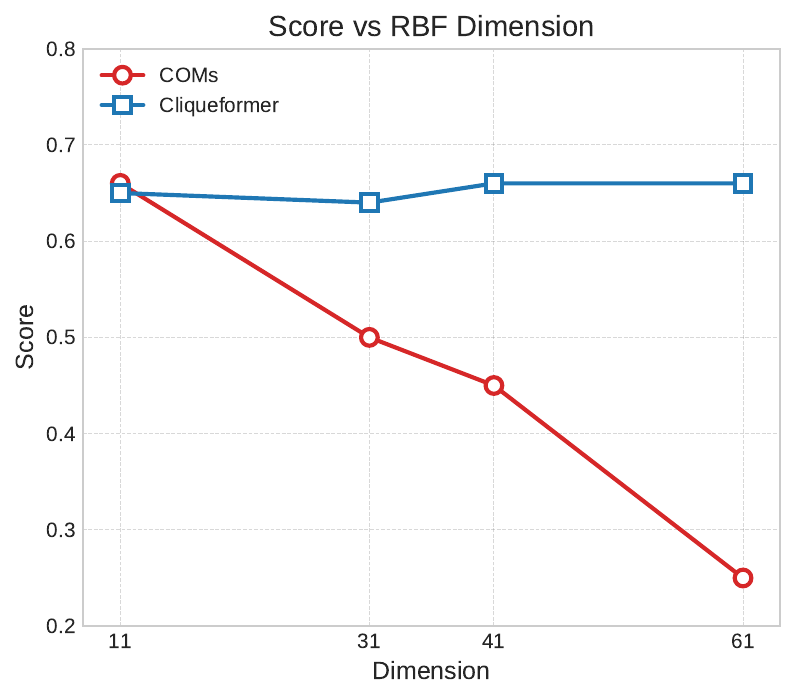}\\[-2pt]
    \footnotesize (c) Score vs. dimension
  \end{minipage}
    \caption{
        The first building block of the Lat. RBF tasks are 3D radial-basis functions (left). 
        These functions are applied to triplets arranged in a chain of triangles FGM (center) and linearly mixed. 
        Then, observable designs are produced with non-linear transformations of the chain and, together with their values, form a dataset.
        We show the score (right) of our structure-learning Cliqueformer and structure-oblivious COMs \citep{trabucco2021conservative}, against the dimension of Lat. RBF functions, modulated only by varying the number of triangles. 
        Cliqueformer, unlike COMs, sustains strong performance across all dimensions.
        More results in Section \ref{sec:exp}.
    }
    \label{fig:rbf-scores}
\end{figure*}

\section{Experimental details}
\label{ap:exp}
\paragraph{Datasets.} We use the implementation of \cite{grudzien2024functional} to generate data with latent radial-basis functions (see Figure \ref{fig:rbf-scores} for details). 
Also, we initially wanted to use Design Bench \citep{trabucco2022design} for experiments with practical tasks.
However, at the time of this writing, the benchmark suite was suffering a data loss and was not readily available.
To overcome it, we manually found the data and implemented dataset classes. 
TFBind-8 \citep{trabucco2022design} could be fully downloaded since the number of possible pairs $(\rvx, \ry)$ is quite small. 
Hence, a design can be evaluated by looking up its score in the dataset.
For Superconductor \citep{hamidieh2018data}, we pre-trained an XGBoost oracle on the full dataset, and trained our model and the baselines to predict the labels produced by the oracle.
The proposed designs of the tested models are evaluated by calling the oracle as well.
We obtained DNA Enhancers dataset from the code of \cite{uehara2024understanding}, available at 
\begin{center}
    {\small \url{https://github.com/masa-ue/RLfinetuning_Diffusion_Bioseq/tree/master}.}
\end{center}
Following the procedure in 
\begin{center}
    {\small \url{https://github.com/masa-ue/RLfinetuning_Diffusion_Bioseq/blob/master/tutorials/Human-enhancer/1-Enhancer_data.ipynb}.}
\end{center}
we additionally filter the dataset to keep only sequences featured by chromosomes from 1 to 4.
We use their pre-trained oracle for generating labels and evaluation of proposed designs.
Following \cite{fannjiang2020autofocused} and \cite{trabucco2022design}, we train our models on the portions of the datasets with values below their corresponding $80^{th}$.
Upon evaluation, we obtain the ground-truth/oracle value of the proposed design $\ry$, and normalize it as
\begin{align}
    \bar{\ry} = \frac{\ry - \ry_{\min}}{\ry_{\max} - \ry_{\min}}, \nonumber 
\end{align}
and report $\bar{\ry}$. $\ry_{\min}$ and $\ry_{\max}$ are the minimum and the maximum of the training data. 
This normalization scheme is different than, for example, the one in the work by \cite{trabucco2022design}. 
We choose this scheme due to its easy interpretability---a score of $\bar{\ry} > 1$ implies improvement over the given dataset, which is the ultimate objective of MBO methods.
However, we note that a score of less than $1$ does not imply failute of the algorithm, since we initialize our designs at a random sample from the dataset, which can be arbitrarily low-value or far from the optimum.
For some functions, like in latent RBFs, the optima are very narrow spikes in a very high-dimensional space, being nearly impossible to find (see Figure \ref{fig:contour}).
We choose such an evaluation scheme due to its robustness that allows us to see how good ut improving any design our algorithms are overall.

\paragraph{Hyper-parameters.} 
For baselines, we use hyper-parameters suggested by \cite{trabucco2021conservative}. 
We decreased the hidden layer sizes (at no harm to performance) for Lat. RBF 31 and DNA Enhancers tasks where the performance was unstable with larger sizes. 
Also, we haven't tuned most of the Cliqueformer's hyper-parameters \emph{per-task}.
We found, however, as set of hyper-parameters that works reasonably well on all tasks.

On all tasks, we use 2 transformer blocks in both the encoder and the decoder, with transformer dimension of 64, and 2-head attention.
The predictive model $f_{\theta}(\rvz)$ is a multi-layer perceptron with 2 hidden layers of dimension 256. 
We change it to 512 only for DNA Enhancers.
The best activation function we tested was GELU \citep{hendrycks2016gaussian}, and LeakyReLU(0.3) gives similar results.
We use dropout of rate 0.5 \citep{srivastava2014dropout}.
In all tasks, weight of the MSE term to $\tau=$10 (recall Equation (\ref{eq:loss})).
Additionally, we warm up our VIB term linearly for $1000$ steps (with maximal coefficient of 1). 
We train the model with AdamW \citep{loshchilov2017fixing} with the default weight decay of Pytorch \citep{paszke2019pytorch}.
We set the model learning rate to 1e-4 and the design learning rate to 3e-4 in all tasks.
We train the design with AdamW with high rates of weight decay (ranging from 0.1 to 0.5).

In all tasks, we wanted to keep the dimension of the latent variable $\rvz$ more-less similar to the dimension of the input variable $\rvx$, and would decrease it, if possible without harming performance, to limit the computational cost of the experiments.
The dimension of $\rvz$ can be calculated from the clique and knot sizes as 
\begin{center}
    $d_z = d_{\text{knot}} + N_{\text{clique}} \cdot (d_{\text{clique}} - d_{\text{knot}})$.
\end{center}
In most tasks, we used the clique dimension $d_{\text{clique}}=3$ with knot size of $d_{\text{knot}}=1$.
We made an exception for Superconductor, where we found a great improvement by setting $d_{\text{clique}}=21$ and $N_{\text{clique}}=4$ (setting $d_{\text{clique}}=3$ and $N_{\text{clique}}=40$ gives score of 0.99); and DNA Enhancers, where we doubled the clique size (to 6) and halved the number of cliques to (40), to lower the computational cost of attention.
In DNA Enhancers tasks, we additionally increased the MLP hidden dimension to 512 due to greater difficulty of modeling high-dimensional tasks. 
We summarize the task-specific hyper-parameters in Table \ref{table:config}.

We want to note that these hyper-parameters are not optimal per-task. 
Rather, we chose schemes that work uniformly \emph{well enough} on all tasks.
However, each task can benefit from further alteration of hyper-parameters.
For example, we observed that Lat. RBF tasks benefit from different numbers of design steps; for Lat. RBF 41, we found the optimal number to be 400; for Superconductor, it seems to be 200.
Due to time constraints, we have not exploited scalability of Cliqueformer in DNA Enhancers tasks, but observed pre-training losses to decrease more with increased parameter count and training duration.

\begin{table*}[htbp]
\centering
\begin{tabular}{l|c c c c c}
%\hline
\textbf{Task} & \textbf{N\_clique} & \textbf{d\_clique} & \textbf{MLP dim} & \textbf{design steps} & \textbf{Weight decay}  \\
%\hline
\hline
\multirow{1}{*}{\textbf{Lat. RBF 11}} & 10 & 3 & 256 & 50 & 0.5 \\
%\hline
\multirow{1}{*}{\textbf{Lat. RBF 31}} & 18 & 3 & 256 & 50 & 0.5 \\
%\hline
\multirow{1}{*}{\textbf{Lat. RBF 41}} & 20 & 3 & 256 & 50 & 0.5 \\
%\hline
\multirow{1}{*}{\textbf{Lat. RBF 61}} & 28 & 3 & 256 & 50 & 0.5 \\
%\hline
\multirow{1}{*}{\textbf{TFBind-8}} & 4 & 3 & 256 & 1000 & 0.5 \\
%\hline
\multirow{1}{*}{\textbf{Superconductor}} & 4 & 21 & 256 & 1000 & 0.5 \\
%\hline
\multirow{1}{*}{\textbf{DNA Enhancers}} & 40 & 6 & 512 & 1000 & 0.1 \\
%\hline
\end{tabular}
\caption{Hyper-parameter configuration for different benchmark tasks.}
\label{table:config}
\end{table*}

\begin{comment}
\section*{Links}
We provide an overview of our work on the project website at 
\begin{center}
    {\small \url{https://sites.google.com/berkeley.edu/cliqueformer/home}.}
\end{center}
In particular, we release our code at 
\begin{center}
    {\small \url{https://github.com/znowu/cliqueformer-code}}.
\end{center} 
\end{comment}

\section*{Links}
We provide an overview of our work on the project website at 
\begin{center}
    {\small \url{https://sites.google.com/berkeley.edu/cliqueformer/home}.}
\end{center}
In particular, we release our code at 
\begin{center}
    {\small \url{https://github.com/znowu/cliqueformer-code}}.
\end{center}

\end{document}